\documentclass[]{elsarticle}
\usepackage{amsthm}
\usepackage[small]{caption}
\usepackage{graphicx}

\usepackage[vlined,linesnumbered,boxed]{algorithm2e}

\newtheorem{theorem}{Theorem}
\newtheorem{lemma}[theorem]{Lemma}
\newtheorem{corollary}[theorem]{Corollary}

\begin{document}
\begin{frontmatter}

\title{Grouping and Recognition of Dot Patterns with Straight Offset Polygons}
\author{Toshiro Kubota}
\ead{kubota@susqu.edu}

\address{Department of Mathematical Sciences\\
Susquehanna University\\Selinsgrove PA 18970, USA}

\begin{abstract}
When the boundary of a familiar object is shown by a series of isolated dots, humans can often recognize the object with ease. This ability can be sustained with addition of distracting dots around the object. However, such capability has not been reproduced algorithmically on computers.
We introduce a new algorithm that groups a set of dots into multiple non-disjoint subsets. It connects the dots into a spanning tree using the proximity cue. It then applies the straight polygon transformation to an initial polygon derived from the spanning tree. The straight polygon divides the space into polygons recursively and each polygon can be viewed as grouping of a subset of the dots. The number of polygons generated is O($n$). We also introduce simple shape selection and recognition algorithms that can be applied to the grouping result. We used both natural and synthetic images to show effectiveness of these algorithms.
\end{abstract}
\end{frontmatter}


\section{Introduction}\label{intro}
Consider a picture shown in Figure \ref{fig:dolphin_with_dots}(a). We can easily recognize a dolphin in the picture, delineated by a series of isolated dots. Reproducing this capability on a computer is not difficult for this simple case. The task quickly becomes difficult when distracting noise dots are added to the picture as shown in Figure \ref{fig:dolphin_with_dots}(b) and (c). We can still recognize the dolphin in (b) and possibly (c). However, most clustering algorithms are not capable of dealing with them as the background noise significantly overlaps with the shape. Most perceptual organization algorithms also have difficulties in dealing with such data, as dots lack any orientation and directional information needed for many algorithms\cite{elder:curve,Mahamud:PO,Wang:RatioContour}.

\begin{figure}
\centering
\includegraphics[width=3in]{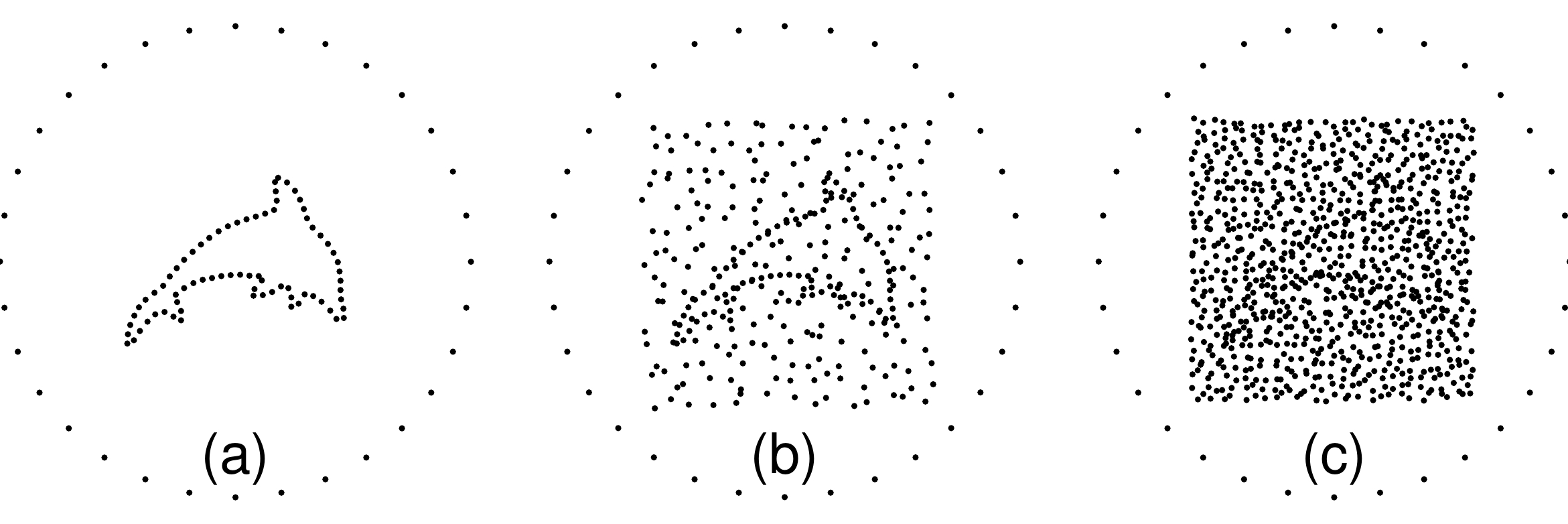}
\caption{Dot patterns} \label{fig:dolphin_with_dots}
\end{figure}

Patterns comprised of isolated dots (often called \textit{dot patterns}) have played important roles in various psycho-visual studies. The studies suggest bottom-up grouping and shape extraction capability in our visual system \cite{Kubovy:PS1995,Greene:BBF2008}. Although they are not typical patterns we encounter in every day experience, they carry raw information that is essential to our visual processing. Thus, it is important to understand how such patterns can be processed algorithmically and how an underlying salient shape can be extracted. Various attempts have been made in the past to divide a dot pattern into coherent sets (clustering problem) \cite{Zahn:IEEE1971,Rosenberg:1973,Ahuja:PAMI1982,Toussaint:PR1980} or delineate the boundary enclosing the pattern (external shape problem)\cite{Edelsbrunner:Shape1983,ORourke:1987,Chaudhuri:CVIU1997,Melkemi:PR2000}. However, the past approaches were unable to deal with overlapping distracters as seen in Figure \ref{fig:dolphin_with_dots}(b), highly dependent on critical parameters\cite{AhujaTuceryan:CVGIP1989,Edelsbrunner:Shape1983}, and limited to a single interpretation of data\cite{Rosenberg:1973}.

In this paper, we propose an algorithm that derives a collection of polygonal regions from a dot pattern in a parameter free manner. It applies minimum spanning tree on the pattern followed by straight polygon transformation of \cite{Aichholzer:skeletons}.
See Figure \ref{fig:simple_example} for an illustration of the idea. Figure \ref{fig:simple_example}(a) is a simple example of a dot pattern. Figure \ref{fig:simple_example}(b) shows a minimum spanning tree derived from (a).As in \cite{Zahn:IEEE1971}, the Euclidean distance between a pair of dots is used as the weight for the spanning tree. Figure \ref{fig:simple_example}(c) shows a straight polygon representation derived from (b). As the polygon deforms outward, a vertex of the polygon touches another part of the polygon. At the time, a new polygon is created by an enclosure of the outward growing polygon. This newly created polygon grows inward as the deformation is contained within the enclosure. Three inward growing polygons are created in the example, which are shown in thick solid lines. An inward growing polygon can further divided into multiple polygons if a concave vertex touches another side of the polygon before it vanishes. As we outline in more detail later in this paper, we can trace back vertices of a new polygon back to vertices of the spanning tree. The resulting set of vertices in the spanning tree forms another polygon using the original point set and provides a grouping instance. Therefore, for each inward growing polygon, we can associate a grouping instance of the point set.

The outward growing polygon keeps growing and encloses these inward growing polygons. We can also trace vertices of the outward growing polygon back to the spanning tree. The trace may not form a polygon as it can visit the same edge more than once (in opposite directions). For example, for the polygon pointed by an arrow in Figure \ref{fig:simple_example}(c), the trace provides two polygons joined by an edge that were traversed twice in both directions. One of the two polygons encloses the two rectangles in the upper part and the other encloses the polygon with five vertices in the lower part. Thus, a trace of outward growing polygon can potentially provides multiple grouping instances where each grouping can possibly combine multiple inward growing polygons.

\begin{figure}
\centering
\includegraphics[width=3in]{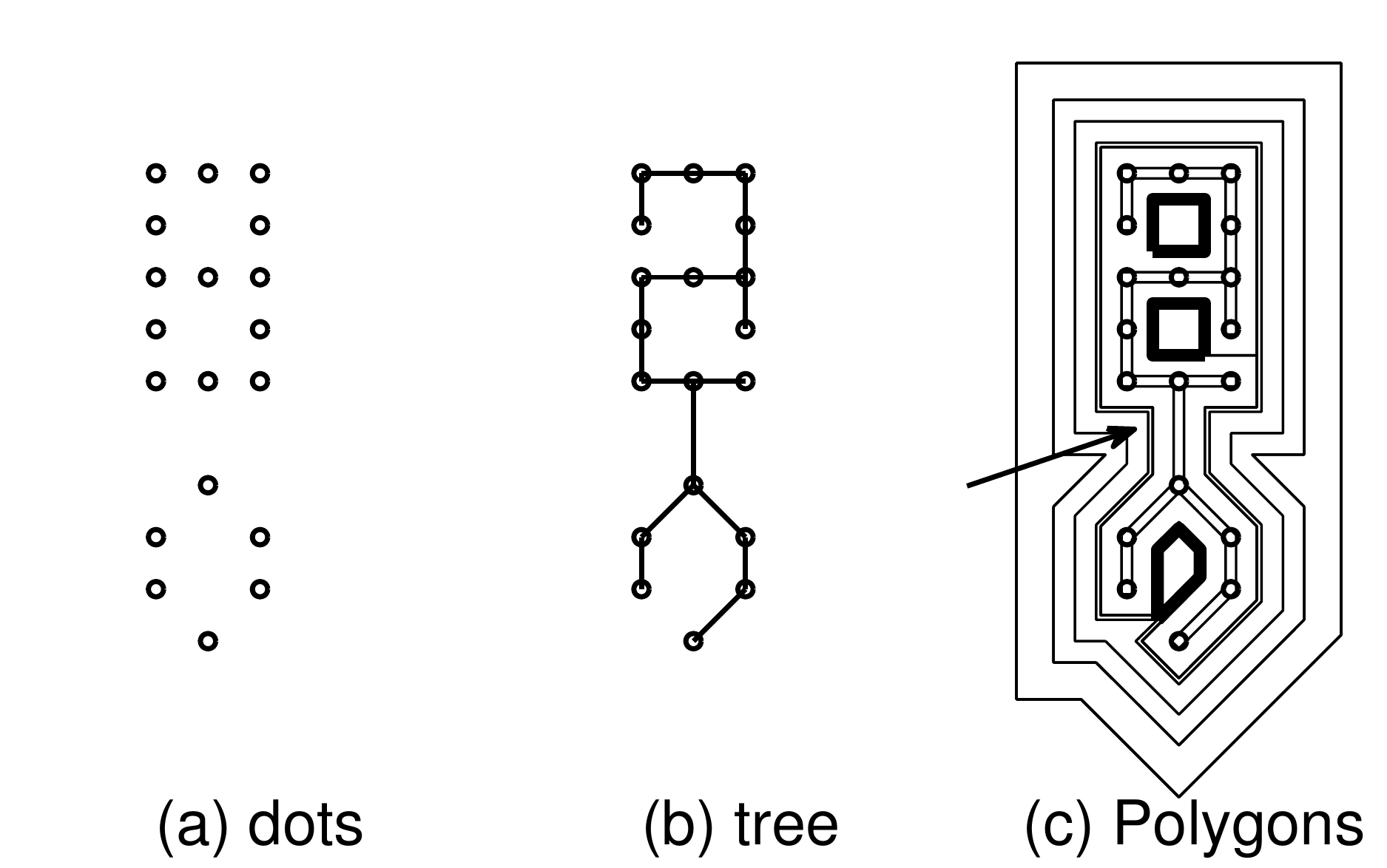}
\caption{A simple illustration of our grouping algorithm.} \label{fig:simple_example}
\end{figure}

In summary, the straight polygon transformation applied to a minimum spanning tree of a dot pattern can provide multiple grouping hypotheses in a simple deterministic manner without any parameters. Our experiments presented in this paper show that the approach offers grouping performance that is robust against noise and grouping hypotheses that are agreeable to our perception. The approach complies with multiple interpretations and multiple solutions characteristics advocated in \cite{Engbers:PAMI2003} for a good grouping algorithm.

The rest of the paper is organized as follows. Section 2 provides general overview of the past research related to this current work. Section 3 describes the straight polygon deformation and straight skeleton representation. Section 4 formalize our grouping approach as outlined above, followed by a simple approach to select salient representative polygons. Section 5 provides empirical evaluation results. Section 6 concludes the paper with a brief summary and future directions.
An earlier preliminary version of the paper appeared in \cite{Kubota:VISAPP2015}.

\section{Related works}
This section provides a brief review of past research on dot pattern perception and grouping.

Kubovy investigated grouping of dots arranged on a periodic rectilinear lattice and studied stability of the grouping based on proximity \cite{Kubovy:1994PBR}. The multi-stability of dot grouping was later modeled with an exponential decay function of the distance between nearby pairs\cite{Kubovy:PS1995}. Good fit of the model advocates proximity based groping. However, in \cite{StrotherKubovy:Curvature2006}, dots arranged on a curvilinear lattice showed grouping over smooth curves instead of over more proximal straight curves. The study appears to contradict the pure proximity based model and suggests more complex interplay between proximity and curvature cues in our perception. Nevertheless, the patterns used in \cite{StrotherKubovy:Curvature2006} are rather unnatural as dots are aligned perfectly along parallel curves providing distinct texture and they are viewed through a small aperture without any notion of boundaries. Hence, we think that the proximity cues are more dominant than the curvature ones, and stay focused on the proximity cues alone in our current study.

Greene has conducted various studies on shape recognition using a device that allows controlling of a display of a 64x64 array of LEDs at sub-millisecond accuracy. In \cite{Greene:PMS2007}, a sparse set of dots uniformly sampled around the shape induced recognition of the shape more quickly than another set distributed non-uniformly around the shape. The result suggests that the maximum separation of dots affects the speed of shape recognition. In \cite{Greene:BBF2008}, using the same LED display, dots delineating common shapes were grouped into a group of four dots and they were flashed at millisecond accuracy. One treatment selected the four dots consecutively from the outline, thus provides contour cues. The other treatment selected the four dots randomly, thus depriving any contour cues. Subjects were divided into one of the two treatments and recognition accuracy was recorded for each shape. The results showed that there was no significant difference in the two treatment groups.
The study suggests that contour attributes such as orientation, curvature, and length commonly used for perceptual grouping models are less important than the proximity cue for construction of shape outlines.

Computationally, most of perceptual organization works in recent years focus on grouping of edges. The amount of work on isolated dots is small in comparison. We attribute this unbalance to three factors. First, the importance of boundaries in our perception has been demonstrated and accepted. Edge grouping address the problem more directly than dot processing. Second, various neurophysiological studies, most notably by Hubel and Wiesel \cite{Hubel:1959}, have shown that edges or stimuli with directional and orientational attributes elicit strong responses to early stages of vision systems. Third, edges with directional and orientational cues allow more elaborate computational models than non-directional and non-orientational dots.

There are two computational issues with the dot grouping problem: representation and clustering. The former is to bestow neighborhood relations to the set of dots, and the latter is to group them based on some criteria derived from the relations. Voronoi diagram (or equivalently Delaunay triangulation) has been used for the representation in \cite{Ahuja:PAMI1982}\cite{AhujaTuceryan:CVGIP1989}\cite{Kubota:Shape2013}. Fully connected \cite{Williams:NECO97a}\cite{guy:percgrp} and K-nearest-neighbor have also been considered for providing neighbor relations to the unstructured set of dots\cite{Jarvis:1977}. Toussaint proposed a \textit{relative neighborhood graph}\cite{Toussaint:PR1980}, in which two points $p$ and $q$ are connected if the intersection of two spheres centered at $p$ and $q$, respectively, with the radius of $\|p-q\|$ contains no other points in the dot pattern. The intersection of the two sphere forms a \textit{luna}. The relative neighborhood graph is a superset of a minimum spanning tree and a subset of a Delaunay triangulated graph. Similarly, in a Gabriel graph, two points $p$ and $q$ are connected if a sphere centered at the mid point of $p$ and $q$ with the radius of $\|p-q\|/2$ contains no other points\cite{Gabriel:1969}. These two graphs tend to retain edges that are perceptually salient. However, they do not provide interpretation of how dots can be clustered into a collection of shapes. Simple elimination of edges based on proximity cues does not appear sufficient\cite{Toussaint:1988}\cite{Papari:2005}.

%

When a dot pattern represents a single cluster, the problem is to derive a polygonal representation of the cluster. The problem is often called \textit{external shape extraction}. A trivial but important representation is a convex hull. Edelsbrunner et al. proposed a generalization of convex hull called $\alpha$-\textit{hull}. Given a real number $\alpha\in(-\infty,\infty)$, the $\alpha$-\textit{hull} is the intersection of all closed generalized discs with radius $1/\alpha$ that contain all the points in the pattern \cite{Edelsbrunner:Shape1983}. If $\alpha<0$, a generalized disk is the complement of a disc of radius $-1/\alpha$ and if $\alpha=0$, it is a half-plane.  The convex hull is a case with $\alpha=0$. Furthermore, $\alpha$-\textit{shape} is a polygonal representation of the dot pattern derived from the corresponding $\alpha$-\textit{hull} and can be computed in $O(n \log n)$. Chaudhuri et al. proposed a representation called $r$-\textit{shape}, which is simpler and computationally more efficient ($O(n)$) than the $\alpha$-\textit{shape} \cite{Chaudhuri:CVIU1997}.

For the clustering side, standard clustering algorithms such as k-means, mixture of Gaussian, and ISODATA can be applied. However, these algorithms work on compact clusters that are well separated, and do not work on the dolphin example shown in Figure \ref{fig:dolphin_with_dots}. Parametric or template based models \cite{Olson:IVC1998} can be used to isolate specific shapes from background. However, the approaches are not applicable to general shapes.

Zahn used a minimum spanning tree from a dot pattern, as we do in our algorithm, and break the tree into a forrest by removing edges that are significantly longer than the others \cite{Zahn:IEEE1971}. The method of Bajcsy and Ahuja \cite{Bajcsy:PAMI1998} is similar to Zahn's method, but exploits maximum intra-cluster similarity and inter-cluster dissimilarity.
Ahuja and Tuceryan \cite{AhujaTuceryan:CVGIP1989} used the Voronoi diagram to derive the neighbor relation and various local geometric structures from the diagram to classify dots into interior, border, curve, and isolated. Globally consistent classification is encouraged by relaxation labeling\cite{Rosenfeld:relax}\cite{Hummel:Relax}. They use 7 different geometric structures and the classification method is complex and requires many free parameters.

\section{Straight offset polygons and straight axes}\label{rep}
In this section, we describe straight offset polygon and straight skeleton representations as introduced in \cite{Aichholzer:skeletons} and extended in \cite{Aichholzer:skeletons2}. Let \textit{P} be a polygon with \textit{n} vertices. The process of forming the straight offset polygon representation is to shrink the polygon by moving inward each side of the polygon by self-parallel motion. Such motion can be generated by moving each vertex to the direction of the angle bisector with the velocity given by
\begin{equation}\label{eq:velocity}
v_i\propto 1/\sin(\theta_i/2)
\end{equation}
where $v_i$ is the velocity of the $i$th vertex and $\theta_i$ is the angle of the polygon at the $i$th vertex.
Two events can change the shape and topology of the polygon: edge event and split event. An edge event occurs when two adjacent vertices collide and changes the shape of the polygon. A split event occurs when a concave (or reflexive) vertex collide with a side of the polygon and split the polygon into two.

See Figure \ref{fig:simple_offset_polygons} for an illustration of these events. It shows a polygon with 8 vertices represented by small hollow circles. The polygon undergoes shrinkage by self-parallel motion, and snapshots of the shrinkage are shown with dashed lines. First, a split event takes place when the vertex pointed by an arrow collides with a side of the polygon. The event splits the polygon into two polygons: one with 5 vertices (left polygon) and the other with 4 vertices (right polygon). These new polygons undergo the same shrinkage process independently. Each polygon experiences a number of edge events where two adjacent vertices collide. There are 3 edge events in both left and right polygons. The shrinkage stops when the polygon vanishes to a line or a point. In Figure \ref{fig:simple_offset_polygons}, polygons at split or edge events are shown in solid lines with the exact locations of events shown with filled circles.

\begin{figure}
\centering
\includegraphics[width=3in]{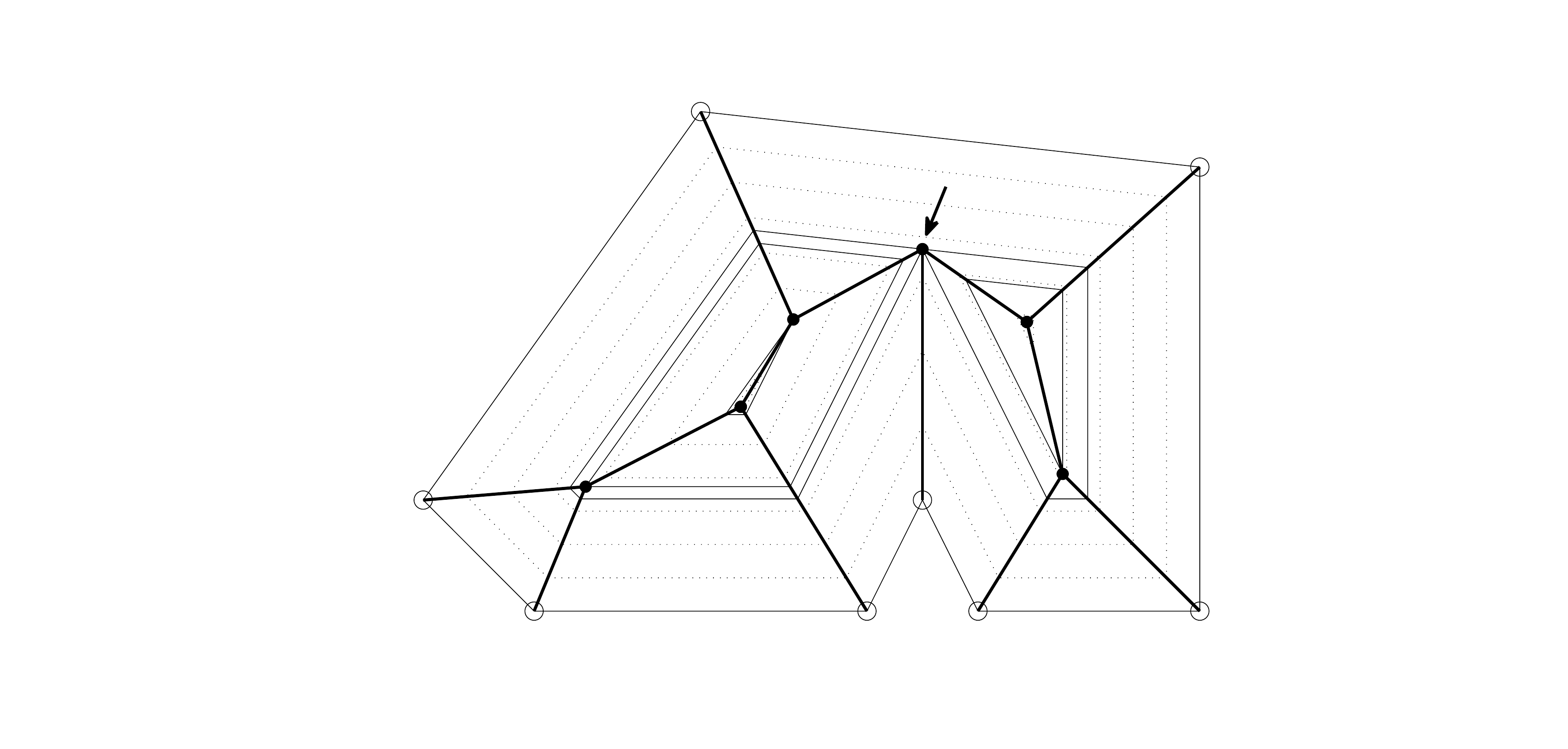}
\caption{An example of straight offset polygons.} \label{fig:simple_offset_polygons}
\end{figure}

By tracing vertices of offset polygons, we obtain a tree-structure, which is called \emph{straight skeletons} of the polygon. The skeletons are shown in thick solid lines in Figure \ref{fig:simple_offset_polygons}. Without any degenerate cases where more than two points collide simultaneously or two parallel sides collide, there are $n-2$ non-leaf nodes in the straight skeletons and $2n-3$ arcs. The skeletons also divide the original polygon into $n$ faces.

Instead of shrinkage, we can consider expansion of the polygon by self-parallel motion of each side. This can be achieved by simply reversing the motion of each vertex. The polygon deforms and its shape and topology change as edge and split events occur as in the shrinkage case. Some vertices do not vanish and approach infinity.

As the name suggests, the straight skeletons are comprised of linear line segments. In contrast, the medial axes are comprised of linear and quadratic line segments in general\cite{Blum:1967}. Thus, the data structure for the straight skeletons is simpler than that of the medial axes. However, the straight skeletons lack a dual interpretation of Voronoi diagram as in the medial axes. As a result, we need to simulate the shrinkage (or expansion) process to obtain the straight skeletons. The run-time of a brute force implementation of the simulation is in $O(n^3)$  where $n$ is the number of vertices, but in practice, it runs in $O(n^2)$. The medial axes can be constructed in $O(n\log n)$ \cite{Lee:PAMI1982}. The number of events is upper bounded by $n-2$ (Lemma \ref{lemma:numEvents}). At each event, at most two polygons are created. Thus, the number of polygons is upper bounded by $2n-4$.

Algorithm \ref{algo:offset_polygons} shows procedures that implement the construction of the straight offset polygons. The initialization (Lines 2-4) takes $O(n^2)$ since the for-loop runs over vertices, and for each vertex, we need to examine other vertices and sides for the earliest incidence of events. The recursion ($apply$) is called for each event, which is $O(n)$. Handling an event (Lines 11-17) takes a constant time. Some events in the queue need to be updated if the current event removed the participating vertex or side for the events (Lines 18-21). The number of events that need to be updated is $O(n)$ but typically is dependent on local shape and thus is $O(1)$. For each event, it takes $O(n)$ to update. Therefore, the $apply$ procedure runs in $O(n^2)$ for the worst case but $O(n)$ for typical cases. Overall, the algorithm runs in $O(n^3)$ for the worst case but $O(n^2)$ for typical cases. In this paper, we call straight offset polygons and straight skeletons combined \textit{straight polygon representation} and the process of computing the straight polygon representation \textit{straight polygon transformation}.

\begin{algorithm}[htp]
  \SetAlgoLined\DontPrintSemicolon
  \SetKwFunction{main}{OffsetPolygons}
  \SetKwFunction{rec}{apply}
  \main{\textbf{P}}{
  \KwIn{\textbf{P}: a polygon, i.e. a sequence of vertices}
  \KwOut{\textbf{G}: a set of polygons}
    \ForEach{$q$ in \textbf{P}}
    {
        find the next event for $q$\;
    }
  \textbf{G} = \rec{\textbf{P}}\;
  \KwRet \textbf{G}\;
  }{}
  \rec{\textbf{P}}{
  \KwIn{\textbf{P}: a polygon}
  \KwOut{\textbf{G}: a set of polygons}
  \textbf{G}$\leftarrow\emptyset$\;
  \If{$|\mathbf{P}|\geq 3$}
  {
    extract the next event from \textbf{P}\;
    \If{edge event}
    {
        \textbf{L} = polygon after collision\;
        \textbf{R}=$\emptyset$\;
    }
    \ElseIf{split event}
    {
        \textbf{L},\textbf{R} = polygons after split\;
    }
    \textbf{Q} = a set of vertices affected by the event\;
    \ForEach {$q$ in \textbf{Q}}
    {
        update the next event for $q$\;
    }
    $\textbf{G}\leftarrow \left\{\textbf{P}\right\}+\rec{\textbf{L}}+\rec{\textbf{R}}$\;
  }
  \KwRet \textbf{G}\;
  }
  \caption{Procedure for straight polygon transformation}\label{algo:offset_polygons}
\end{algorithm}

\section{Algorithm}\label{algo}
In this section, we first describe our algorithm of computing a collection of grouping of dots from a dot pattern. We then describe an efficient clustering algorithm that can be used to cluster the collection of polygons and reduce them into a selected few representative ones. We call the former grouping algorithm and the latter selection algorithm. We call application of the grouping algorithm followed by the selection algorithm as a figure extraction algorithm.

\subsection{Grouping}\label{grouping}
Algorithm \ref{algo:representation} shows an outline of the algorithm. First a minimum spanning tree is constructed from the input dot pattern (Line 1). As in \cite{Zahn:IEEE1971}, the Euclidean distance between a pair of dots is used as the edge weight. The tree is traced to construct a circularly linked list of polygon vertices with the velocity given by  (\ref{eq:velocity}) (Line 2). Figure \ref{fig:InitialPolygonTrace} shows examples of the trace. As shown in Figure \ref{fig:InitialPolygonTrace}(a), when there is no branch in the spanning tree (\textit{path tree}), the trace gives a polygon of $2n$ vertices where $n$ is the number of nodes in the spanning tree. As shown in Figure \ref{fig:InitialPolygonTrace}(b), when one tree node is adjacent to every other node in the tree (a \textit{star tree}), the trace gives a polygon of $3n-3$ vertices. The number of polygon vertices are bounded by these two extreme cases (Corollary \ref{corollary:numVertices}). The polygon does not need any particular width as long as all vertices are prepared with the velocity defined by (\ref{eq:velocity}). From the initial polygon, the straight polygon transformation is applied with Algorithm \ref{algo:offset_polygons} (Line 3). For each moving vertex, additional attributes are added as described next. To help describing the algorithm, let $\Phi$ be a set of vertices in the minimum spanning tree and $\Psi$ be a set of vertices in the initial polygon derived from the minimum spanning tree.

From each offset polygon, we want to trace back to a subset of vertices in $\Psi$ and in turn to $\Phi$. For the purpose, each vertex maintains two links ($\pi_a$ and $\pi_b$). For vertices in $\Psi$, $\pi_a=\pi_b=nil$. Upon an edge event, two vertices collide and a new vertex is formed. $\pi_a$ and $\pi_b$ of the new vertex are set to the colliding vertices in such a way that $\pi_b$ is the successor of $\pi_a$ in the circular linked list of the polygon. Upon a split event, a vertex collides with a side of the polygon and two new vertices are formed. For both of the new vertices, $\pi_a$ is set to the colliding vertex and $\pi_b$ is set to $nil$. Using $\pi_a$ and $\pi_b$, we can associate the sequence of vertices in an offset polygon back to a sequence of vertices in $\Psi$ (Line 5). By maintaining these assignments, the sequence of vertices after tracing back to $\Psi$ forms a polygon.

As illustrated in Section \ref{intro} with Figure \ref{fig:simple_example}, a sequence of vertices obtained in Line 5 may not form a polygon when they are traced back to $\Phi$ as some sequence can trace the same tree edge more than once in opposite directions. In the next step, we trace the sequence in $\Phi$ and decompose it into a disjoint set of subsequences where each form a separate polygon in $\Phi$ (Line 6).

Let $m\in\left[2n, 3n-3\right]$ be the number of vertices in $\Psi$. Note that offset polygons by successive edge events trace back to the same set of vertices in the initial polygon. Thus, they do not provide any new grouping hypotheses. There can be at most $m/2$ split events. This case occurs if polygons of 3 vertices are split successively without any edge events. Each provides a distinct grouping hypothesis. These split polygons can be aggregated by the outgrowing polygon. The aggregation is done hierarchically. Thus, there can be at most $m/2$ aggregation instances. Therefore, there can be at most $m$ grouping hypotheses in total.

\begin{figure}
\centering
\includegraphics[width=2.5in]{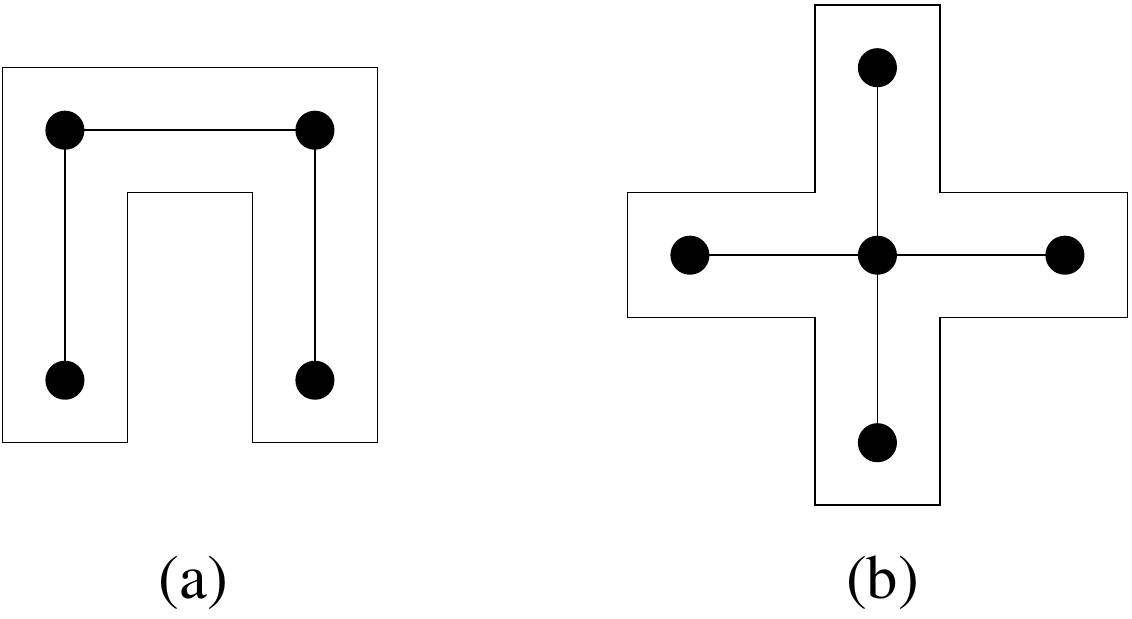}
\caption{Two examples of a spanning tree and an initial polygon.} \label{fig:InitialPolygonTrace}
\end{figure}

\begin{algorithm}[htp]
  \SetAlgoLined\DontPrintSemicolon
  \SetKwFunction{main}{Grouping}
  \main{\textbf{Z}}{
  \KwIn{\textbf{Z}: a set of isolated dots}
  \KwOut{\textbf{H}: a set of polygons}
  \textbf{T} $\leftarrow$ minimum spanning tree of \textbf{Z}\;
  \textbf{P} $\leftarrow$ a thin polygon tracing around \textbf{T}\;
  \textbf{G} $\leftarrow$ OffsetPolygons(\textbf{P})\;
  \ForEach {\textbf{g} in \textbf{G}}
  {
    \textbf{h} = trace \textbf{g} to $\Phi$\;
    $\left\{\textbf{J}\right\}\leftarrow$ polygonal regions in \textbf{h}\;
    add each element of $\left\{\textbf{J}\right\}$ to \textbf{H}\;
  }
  \KwRet{\textbf{H}}\;
  }
  \caption{Procedure for dot grouping}\label{algo:representation}
\end{algorithm}

\subsection{Selection}\label{selection}
Algorithm \ref{algo:representation} brings a collection of polygons. Some polygons are salient while others are not. There are also polygons that are perceptually similar. To integrate the algorithm into practical computer vision applications, we need a way to reduce the number of polygons objectively. In this section, we describe a way to select $K$ perceptually distinct and most salient polygons from the collection of polygons returned by Algorithm \ref{algo:representation}. To achieve the goal, we need a saliency measure of a polygon and a similarity measure of two polygons.

We propose the following for the saliency measure.
\begin{equation}\label{eq:convexity}
c(P)=\frac{\textrm{area}(P)}{\max_i \|p_i-p_{i+1}\|^2}
\end{equation}
where $P=\left\{p_1, p_2,\cdots,p_N\right\}$ is a polygon with $N$ vertices, $\textrm{area}(P)$ is the area of the polygon, and the index in the denominator runs in modulo-N so that it computes the longest polygon side squared. The measure is unit-less and favors larger $N$. When $N$ is fixed, it favors a regular polygon. Between a pair of polygons, we consider one with the larger measure more salient than the other.

We propose the following for the similarity measure between polygons $P$ and $Q$.
\begin{equation}\label{eq:overlap}
o(P,Q) = |P \cap Q|/|P \cup Q|
\end{equation}
where each polygon is treated as a set of vertices, and $|X|$ indicates the cardinality of the set $X$. The measure is relatively fast to compute but may provide a small measure even when two polygons share a large overlapping area and are perceptually similar. Such case occurs when two sets use considerably different sets of points to delineate similar regions. We could employ a measure based on the areas as we did in \cite{Kubota:VISAPP2015}. However, the area based measure is more expensive to compute than the proposed set based one, and it provided little difference to data sets we tested.

Using the overlap measure, we can cluster polygons into disjoint connected components. More specifically, we merge two polygons if their overlap measure is less than a threshold, $\eta$. We set $\eta=0.5$ in our experiment. For each component, we use the polygon with the highest saliency measure as a representative. We can then select $K$ components whose representatives have $K$ highest saliency measures.

\section{Experiments}\label{experiments}
This section presents results of the figure extraction algorithm. The number of polygons returned by the selection algorithm ($K$) is set to 10. The similarity measure threshold ($\eta$) is set to 0.5. We first present results of the shape extraction algorithm applied to dot patterns and show its effectiveness in extracting the underlying shape under noisy conditions. Next, we evaluate the performance objectively by integrating the results to a simple shape recognition procedure. This part demonstrates potential utility of our algorithms to computer vision applications. We then present figure extraction results on dot patterns derived from Canny edges. This part demonstrates applicability of the algorithms to natural images.

\subsection{Dot patterns}\label{dotPatterns}
Our first experiment is to apply the figure extraction procedure to dot patterns comprised of a shape and noise,
Figure \ref{fig:dolphin_with_dots} shows instances of such dot patterns. There are 20 shapes of animals and common objects. Each shape was generated by tracing the boundary of its binary image, keeping every 10th point while discarding the others, and scaling them so that the shape stretched around $200\leq x\leq 800$ and $200\leq y\leq 800$ in the pixel coordinate. For each shape, we imposed three levels of noise, which were generated in the following way. The average spacing between adjacent points in the shape was computed. Denote the average $\mu$. Then the area between $200\leq x\leq 800$ and $200\leq y\leq 800$ was divided into grids of $s\mu$ by $s\mu$ where $s$ controlled the noise level and had a value of 1, 1.5 or 2. Within each grid, a point was placed randomly while keeping clear of $10\%$ margin around the four border (so that no pair of noise points get too close with each other). Thus, $s=1$ gives the highest amount of noise, and $s=2$ gives the least amount of noise. Finally, 32 evenly spaced dots were placed around a large circle centered at $(500,500)$ with the radius of $490$. This circular pattern was intended to provide a frame of reference to human subjects in our psycho-visual experiment described in \cite{Kubota:VISAPP2015}.

Figure \ref{fig:ShapeDotPatterns} shows 80 dot patterns (20 shapes and 4 noise levels) used in our experiments.
\begin{figure}
\includegraphics[width=5in]{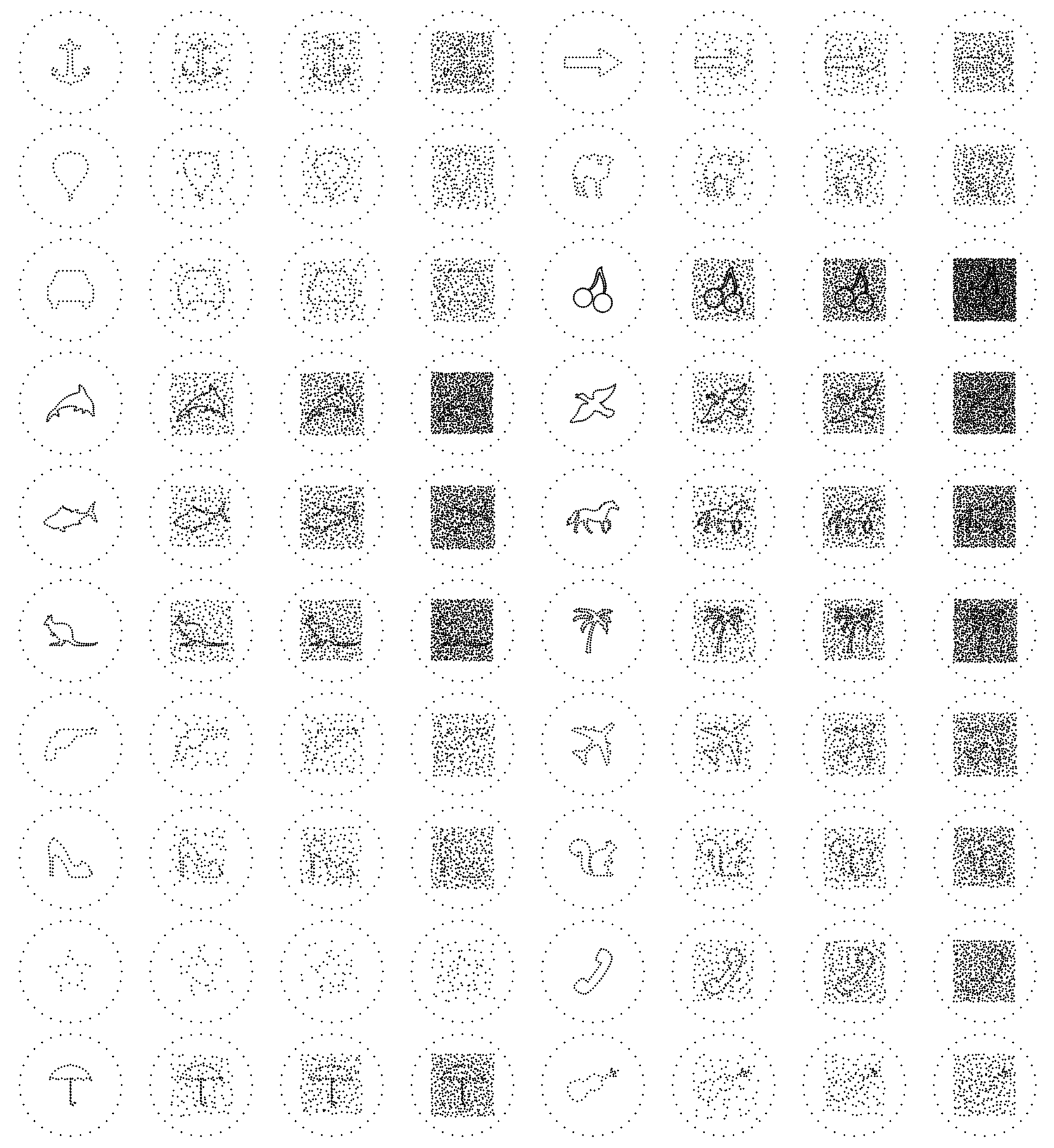}
\caption{Dot patterns of various shapes and noise levels.} \label{fig:ShapeDotPatterns}
\end{figure}

\subsubsection{Grouping}\label{grouping_result}

The figure extraction procedure is applied to each dot pattern shown in Figure \ref{fig:ShapeDotPatterns}.
Results are shown in Figure \ref{fig:ShapeDotPatternResults}. In each plot, K polygons returned by the procedure are drawn. Among them, the one that matches most closely with the underlying shape in the dot pattern is shown in black while the others are shown in gray. We use the following area based similarity measure to find the best match.
\begin{equation}\label{eq:matchingScore}
m(P, Q) = \frac{\textrm{area}(P \cap Q)}{\textrm{area}(P \cup Q)}
\end{equation}
where $P$ and $Q$ are two polygons, and unlike (\ref{eq:overlap}), $\cap$ and $\cup$ of the two polygons are taken as the area of the intersection and the area of the union, respectively. The measure takes a value in $[0,1]$ with 1 when $P$ and $Q$ match exactly and 0 when they are disjoint.

In noise-free instances, the algorithm was able to extract the underlying shape effectively, except that some thin structures (such as the tail of the kangaroo, the tail of the plane, and the stem of the umbrella) were lost. These losses are caused mainly by the saliency measure of (\ref{eq:convexity}) that favors near circular shapes, and partially by the minimum spanning tree based initialization of the polygon as it favors connecting dots across thin structures rather than along the structures. As the amount of noise increases, it becomes more difficult to extract the shapes. However, the algorithm exhibited sustained performance, which correlated well with human perception as demonstrated in \cite{Kubota:VISAPP2015}.

\begin{figure}
\includegraphics[width=5in]{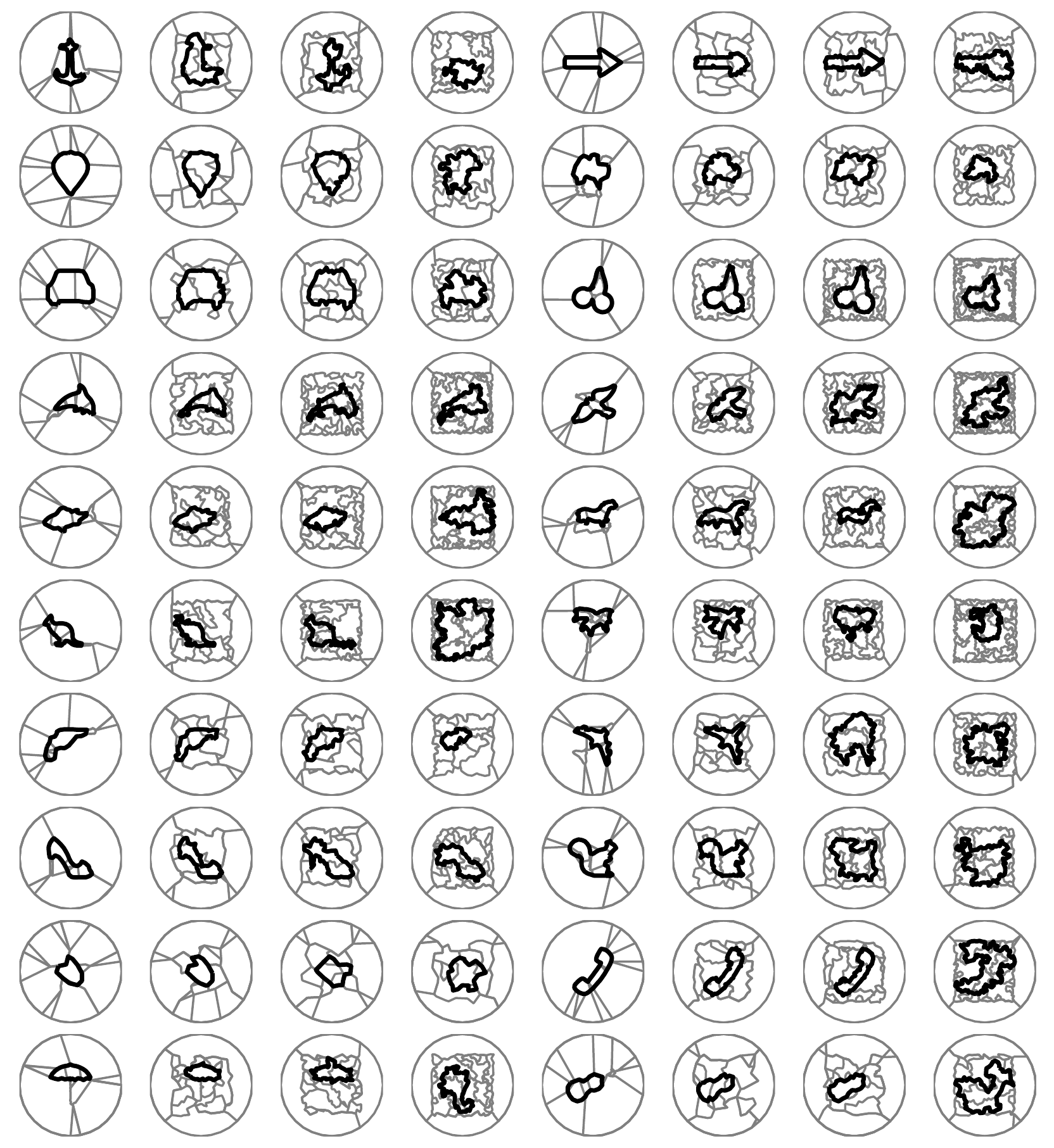}
\caption{Results of grouping applied to the dot patterns of Figure \ref{fig:ShapeDotPatterns}.
10 polygons are drawn on top of each other. The one with the closest match with the shape is shown in black.} \label{fig:ShapeDotPatternResults}
\end{figure}

\subsubsection{Retrieval}\label{retrieval_result}
In this section, we quantitatively evaluate the effectiveness of the figure extraction algorithm by applying shape recognition to the extracted polygons.  Each polygon is used as a query to retrieve the best matched shape in the database. 80 shapes shown in Figure \ref{fig:DatabaseShapes} are used to construct the database. The 20 shapes used in Section \ref{grouping_result} are also included in the database. We also add two more shapes: circle and square. The circle shape corresponds to the outer circle shown in every dot pattern of Figure \ref{fig:ShapeDotPatterns}. The square shape corresponds to the outer hull of noise dots perceived in noisy dot patterns.

We use the area based similarity measure of (\ref{eq:matchingScore}) to find the best match. Some queries return the circle or the square shape as the best match. We ignore these queries, as we are interested in extracting animals and common objects found in each dot pattern. Among remaining queries, we choose the one with the best matching score. If the best match agrees with the shape used to construct the dot pattern, we consider the retrieval successful. If the best match was one of the other 79 shapes, then we consider the retrieval unsuccessful. Note that the database contains multiple instances of the same animal/object category. For example, there are two violins, and two kangaroos. For each query, we consider only the one used to construct the dot pattern as the right match. Thus, for a dot pattern with a violin in Figure \ref{fig:ShapeDotPatterns}, we consider only the second violin in Figure \ref{fig:DatabaseShapes} (the bottom row, 7th column) as the right match.

The result of applying this retrieval procedure to the 80 dot patterns shown in Figure \ref{fig:ShapeDotPatterns} is shown in Table \ref{table:recognitionResult}. The rows are shapes and columns are noise levels. The \textbf{noise=0} column is for noise-free instances. The \textbf{noise=1,2,3} columns are for the lowest, medium, and highest noise levels, respectively. Thus, each cell in the table corresponds to a particular dot pattern in Figure \ref{fig:ShapeDotPatterns}. Each cell shows the shape returned by the retrieval procedure and its match score. The result shows sustained performance against noise until the noise level reaches the highest, at which the recognition rate drops significantly. An psychovisual experiment reported in \cite{Kubota:VISAPP2015} shows similar trend in our perception; we are not able to recognize a shape in dot patterns with the highest noise level.

There are three main causes in the retrieval errors. The first one is inaccurate extraction of the shape (either under-extraction or over-extraction). This type is seen in many noisy patterns. However, it is also seen in the noise-free version of the plane dot pattern (7th row, 5th column in Figure \ref{fig:ShapeDotPatternResults}) where the wing part of the plane is extracted and matched with the hockey player in the database. The source of this type of errors is traced to the shape extraction algorithm. The second type is accidental match where a polygon extracted by the algorithm happens to resemble the shape in the database. An example is the noise-free version of the anchor dot pattern (1st row, 1st column). A polygon that closely approximates the anchor shape was returned by the shape extraction procedure. However, another polygon happened to resemble the coffee pot shape and yielded a higher matching score. A more elaborate matching score would alleviate this type of errors.

\begin{figure}
\includegraphics[width=5in]{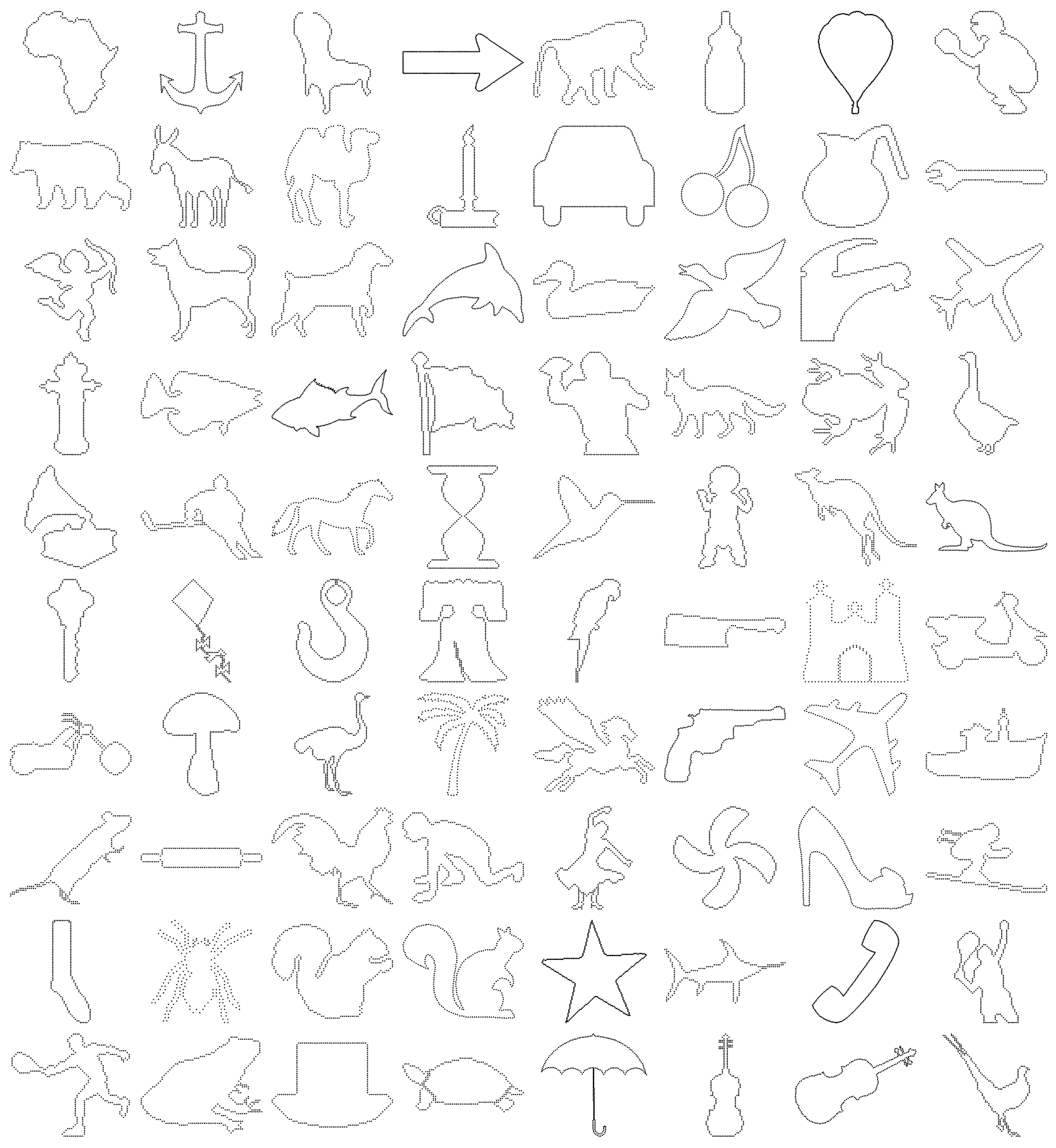}
\caption{80 Shapes in the database.} \label{fig:DatabaseShapes}
\end{figure}

\begin{table}
\centering  \caption{Recognition experiment results. Each row corresponds to the shape used in dot patterns. Each column corresponds to the noise-level with, from left to right, noiseless, low noise, medium noise, and high noise levels, respectively. In each cell, the name of the shape returned by the procedure is shown along with the matching score. Average recognition rate and matching score per noise level are shown at the bottom.}
\small
\begin{tabular}{|c|c|c|c|c|}
\hline
\textbf{shape}	&	\textbf{noise=0}	&	\textbf{noise=1}	&	\textbf{noise=2}	&	\textbf{noise=3}\\
\hline
anchor & coffee pot & coffee pot & coffee pot & top hat \\
	 & 0.725 & 0.562 & 0.512 & 0.542 \\
\hline
arrow & arrow & arrow & arrow & arrow \\
	 & 0.982 & 0.884 & 0.905 & 0.666 \\
\hline
balloon & balloon & balloon & balloon & balloon \\
	 & 0.961 & 0.945 & 0.911 & 0.664 \\
\hline
camel & camel & balloon & camel & car \\
	 & 0.731 & 0.687 & 0.739 & 0.594 \\
\hline
car & car & car & car & car \\
	 & 0.936 & 0.944 & 0.924 & 0.729 \\
\hline
cherries & cherries & cherries & cherries & toad \\
	 & 0.750 & 0.726 & 0.708 & 0.535 \\
\hline
dolphin & dolphin & dolphin & dolphin & dolphin \\
	 & 0.868 & 0.901 & 0.843 & 0.758 \\
\hline
duck & duck & duck & duck & duck \\
	 & 0.901 & 0.795 & 0.732 & 0.663 \\
\hline
fish & fish & fish & fish & turtle b \\
	 & 0.830 & 0.815 & 0.782 & 0.453 \\
\hline
horse & toad & dog b & car & mission \\
	 & 0.692 & 0.608 & 0.598 & 0.555 \\
\hline
kangaroo & kangaroo & kangaroo & kangaroo & baseball catcher \\
	 & 0.832 & 0.784 & 0.737 & 0.408 \\
\hline
palmtree & palmtree & car & car & balloon \\
	 & 0.574 & 0.589 & 0.553 & 0.596 \\
\hline
pistol & pistol & pistol & pistol & pistol \\
	 & 0.752 & 0.694 & 0.685 & 0.429 \\
\hline
plane & hockey player & car & car & top hat \\
	 & 0.446 & 0.659 & 0.596 & 0.690 \\
\hline
shoe & shoe & shoe & shoe & shoe \\
	 & 0.850 & 0.843 & 0.707 & 0.609 \\
\hline
squirrel & squirrel & squirrel & squirrel & mission \\
	 & 0.913 & 0.892 & 0.724 & 0.588 \\
\hline
star & star & bear & star & car \\
	 & 0.633 & 0.611 & 0.663 & 0.619 \\
\hline
telephone & telephone & telephone & telephone & liberty bell \\
	 & 0.982 & 0.943 & 0.938 & 0.442 \\
\hline
umbrella & umbrella & balloon & umbrella & car \\
	 & 0.558 & 0.698 & 0.625 & 0.525 \\
\hline
violin & violin & violin & violin & mission \\
	 & 0.880 & 0.850 & 0.823 & 0.492 \\
\hline
\hline
\textbf{rate}	& \textbf{0.850} & \textbf{0.650} & \textbf{0.800} & \textbf{0.350} \\
\textbf{average}	& \textbf{0.790} & \textbf{0.772} & \textbf{0.735} & \textbf{0.578} \\
\hline
\end{tabular}\label{table:recognitionResult}
\end{table}

\subsection{Edge images}
Our next experiment is to apply the shape extraction to some dot patterns derived from real images. To generate a dot pattern of an image, we first apply Canny edge detector to the image. We then sub-sample edge pixels in each $4\times 4$ block by placing a point at the centroid (rounded to the pixel) of edge pixels in the block. The subsampling reduces the proximity cue and makes the grouping process more difficult. Figure \ref{fig:preprocessing} shows an example of these pre-processing steps where (a) is a gray scale image, (b) is a Canny edge image, (c) is a dot pattern, and (d) is a minimum spanning tree derived from (c).

Figure \ref{fig:edgeResults} shows results of our algorithm applied to dot patterns derived from the Canny edge detector. In each row, there are three set of results. Each set contains three images where, from left to right, the original image, dot pattern, and shape extraction results superimposed on the dot pattern. The 5 most salient polygons (in terms of (\ref{eq:convexity})) are shown.

\begin{figure}
\centering
\includegraphics[width=4in]{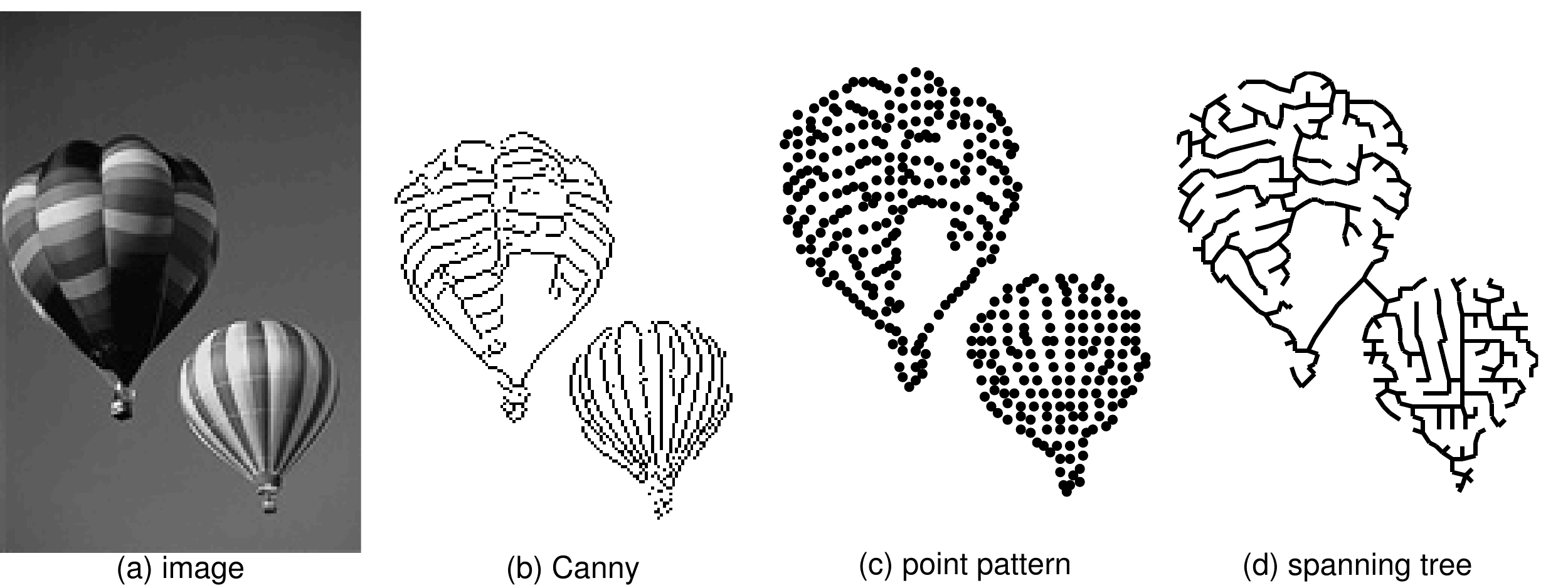}
\caption{Pre-processing steps for a natural image.} \label{fig:preprocessing}
\end{figure}

\begin{figure}
\centering
\includegraphics[width=1.5in]{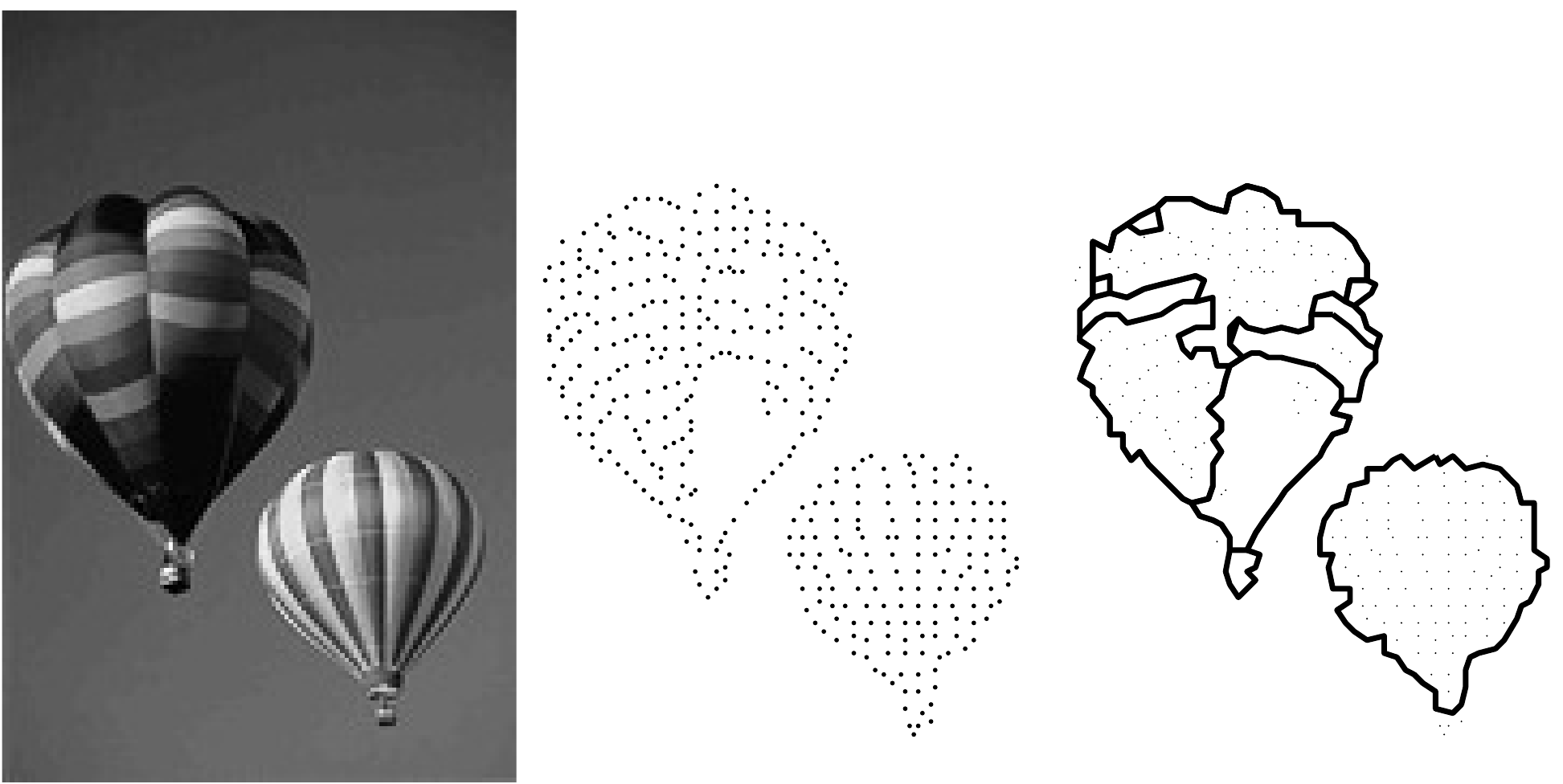}
\includegraphics[width=1.5in]{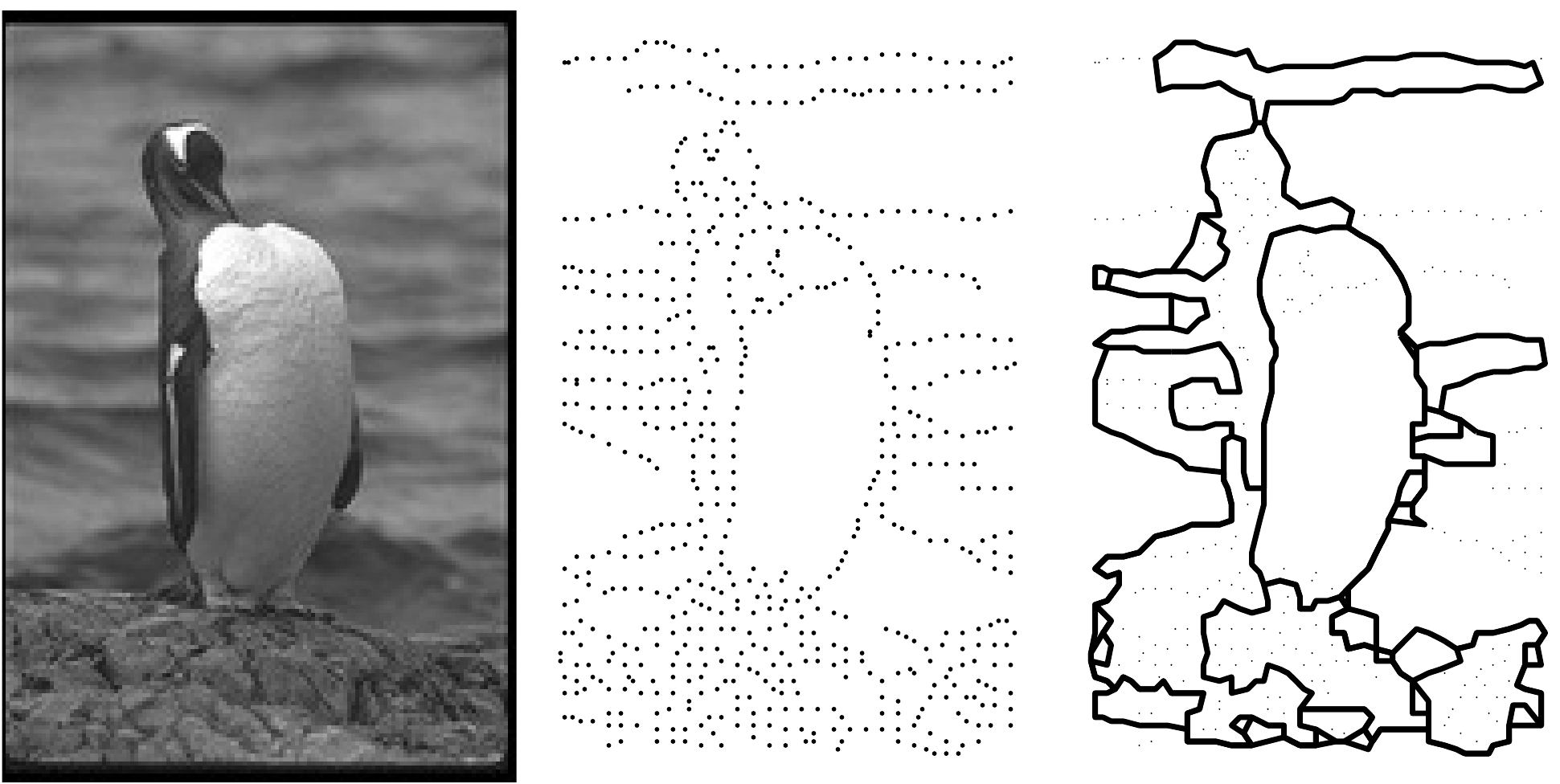}
\includegraphics[width=1.5in]{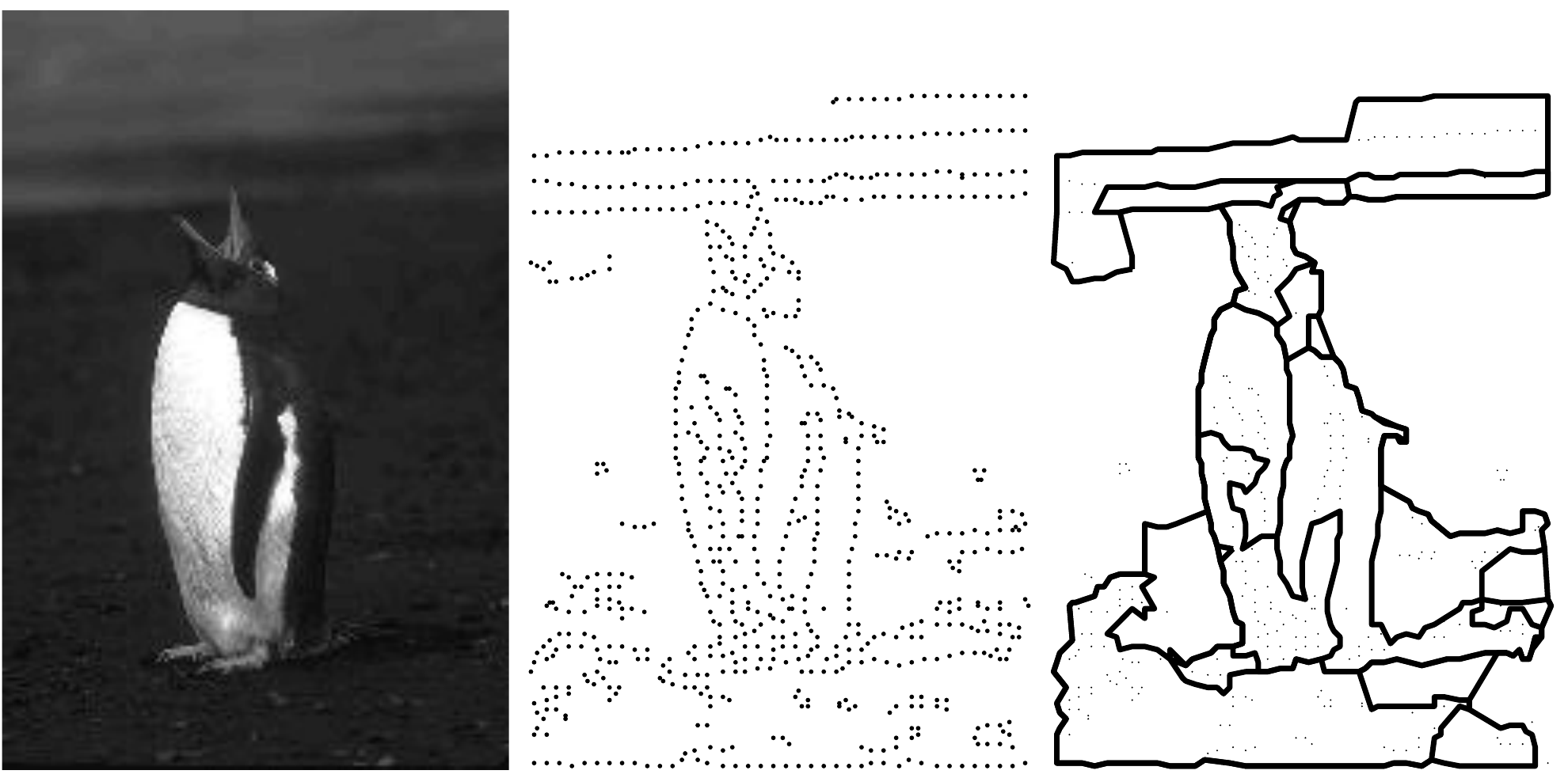}
\includegraphics[width=1.5in]{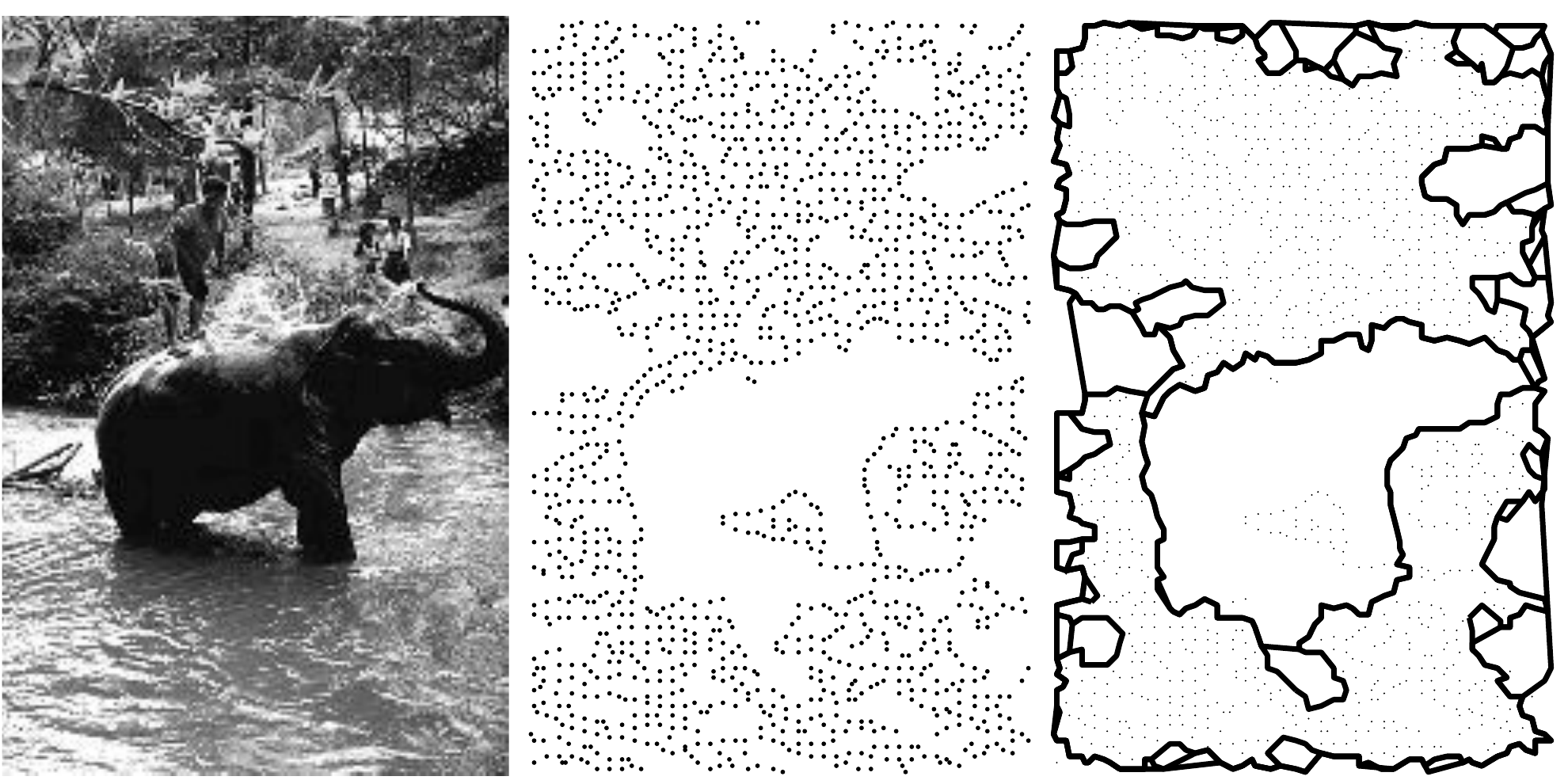}
\includegraphics[width=1.5in]{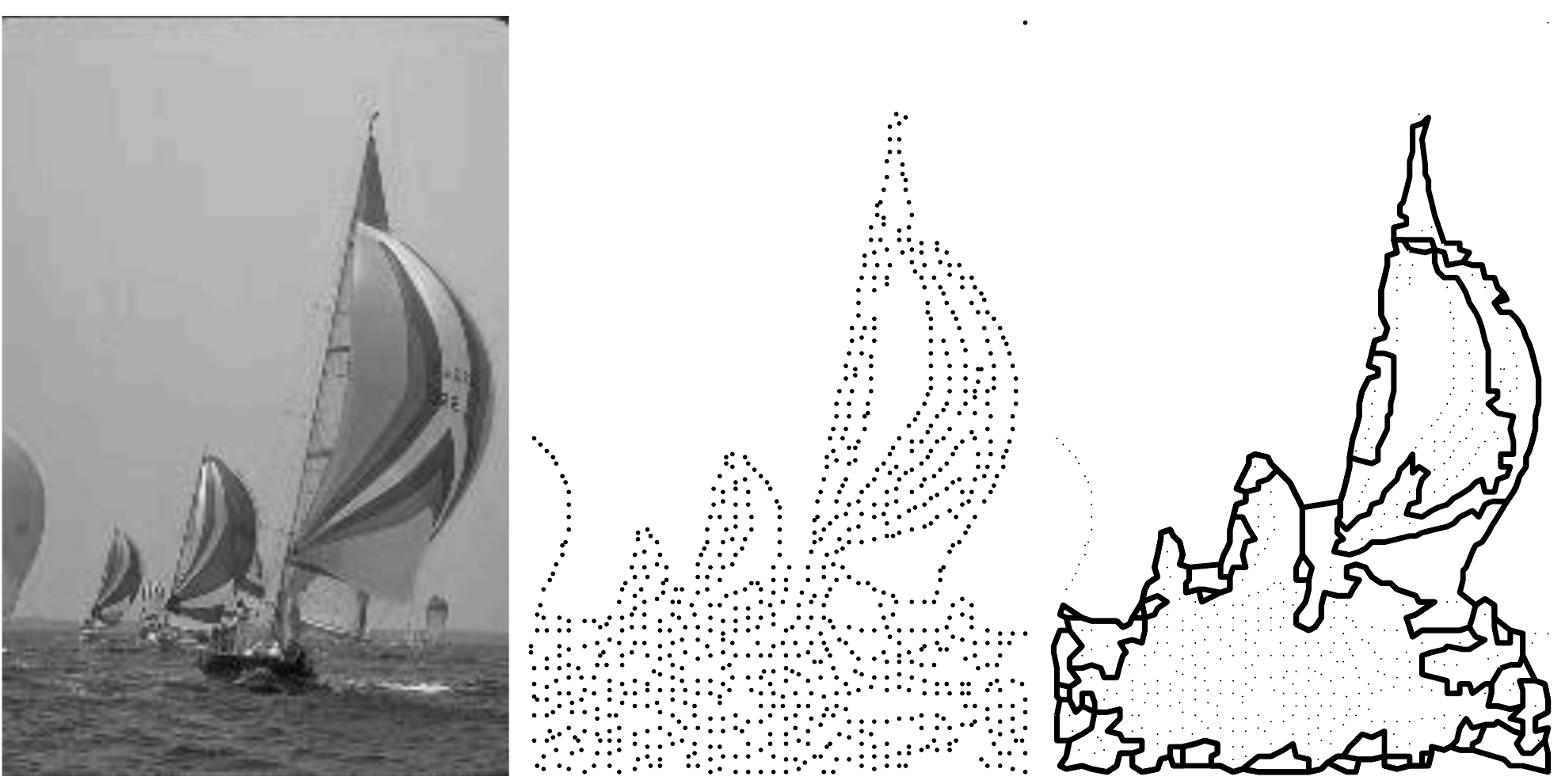}
\includegraphics[width=1.5in]{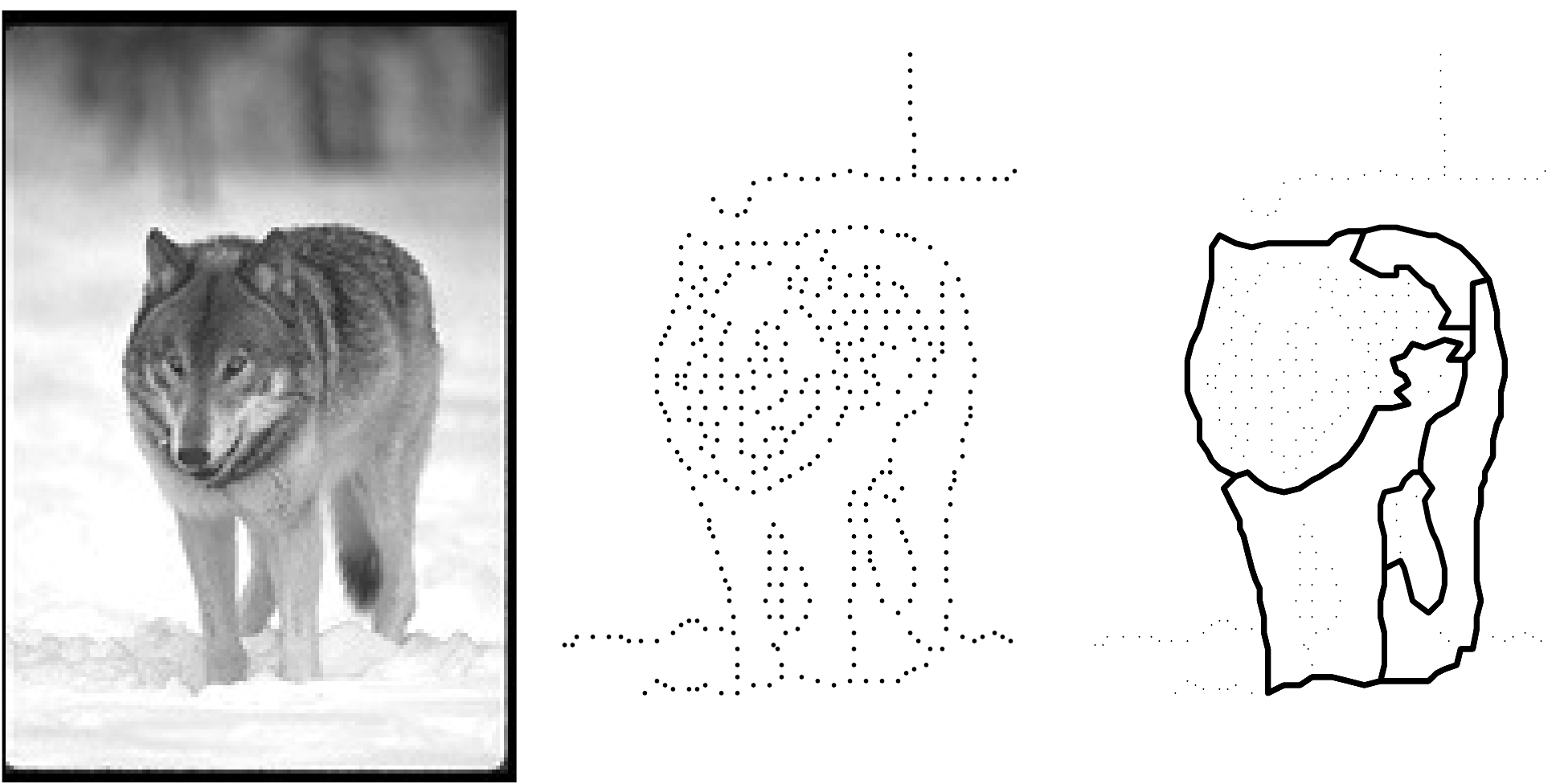}
\includegraphics[width=1.5in]{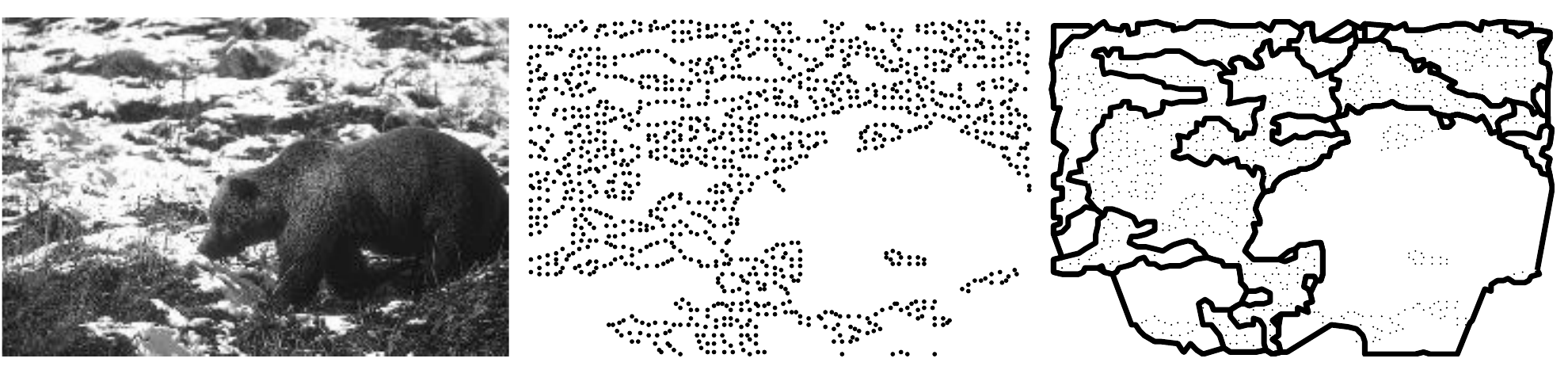}
\includegraphics[width=1.5in]{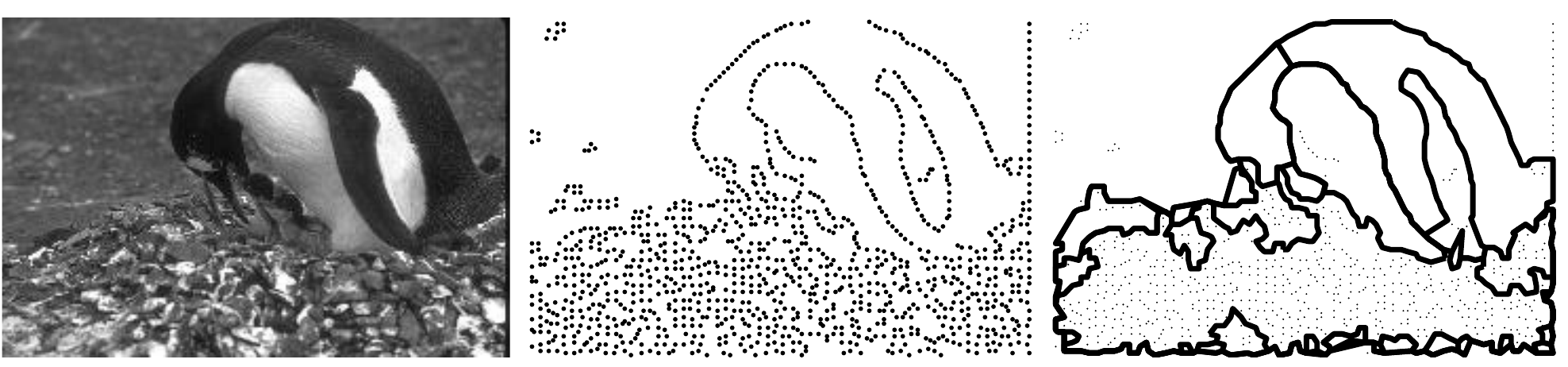}
\includegraphics[width=1.5in]{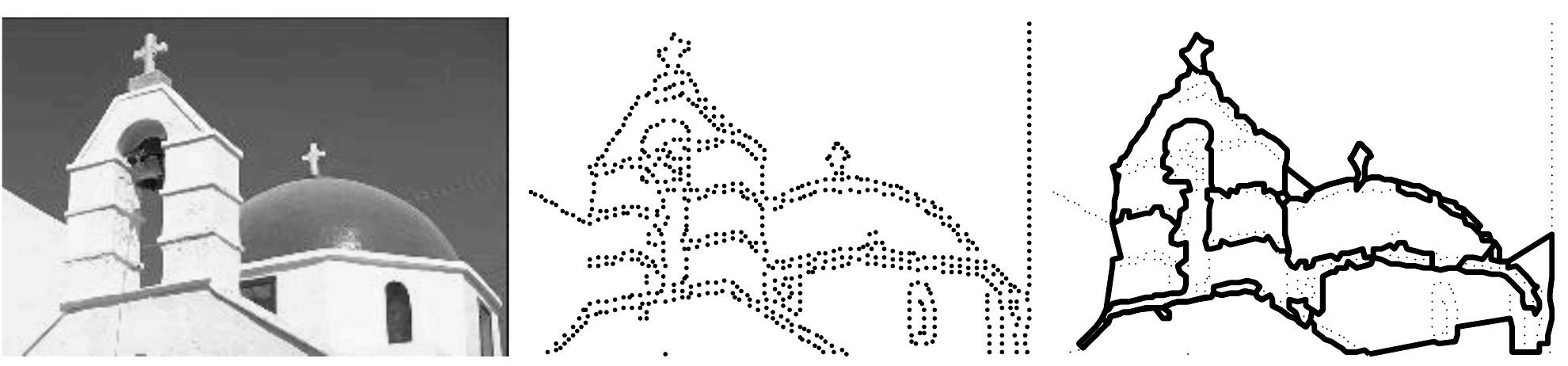}
\includegraphics[width=1.5in]{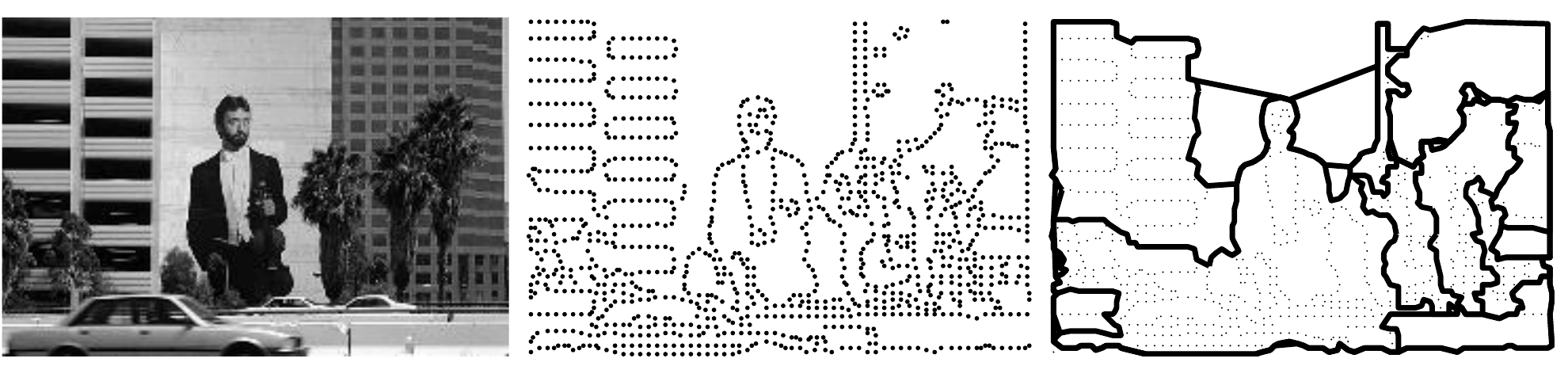}
\includegraphics[width=1.5in]{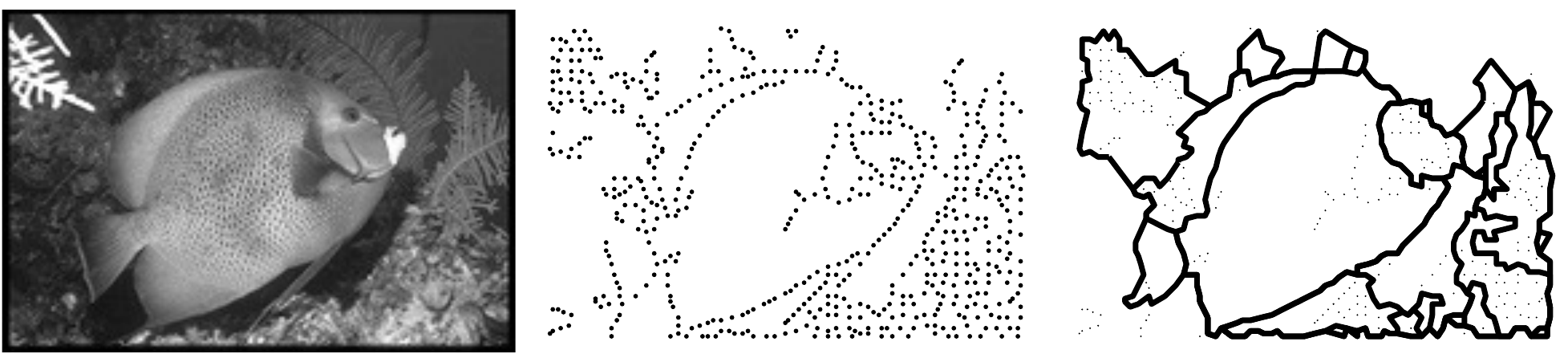}
\includegraphics[width=1.5in]{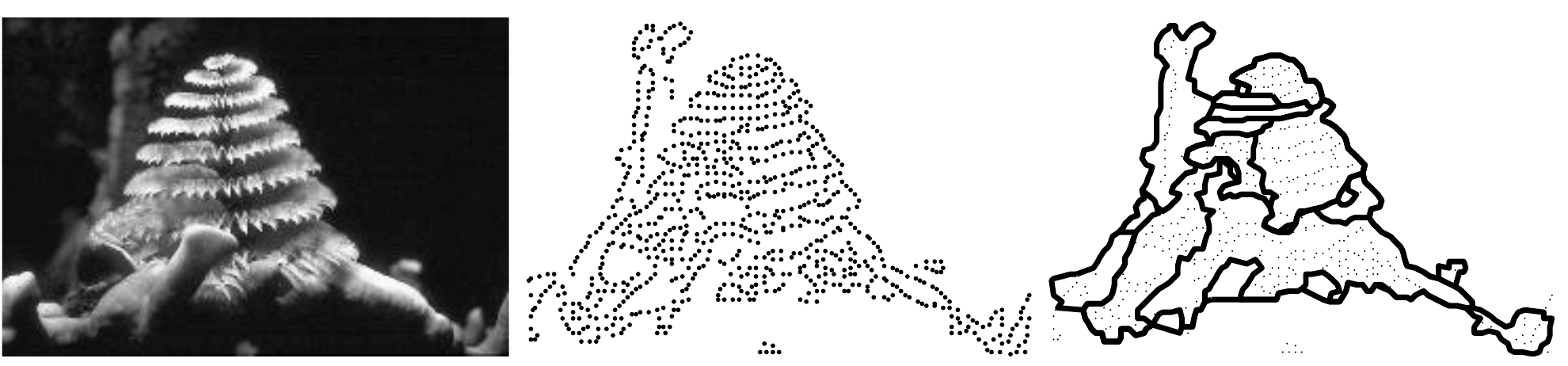}
\includegraphics[width=1.5in]{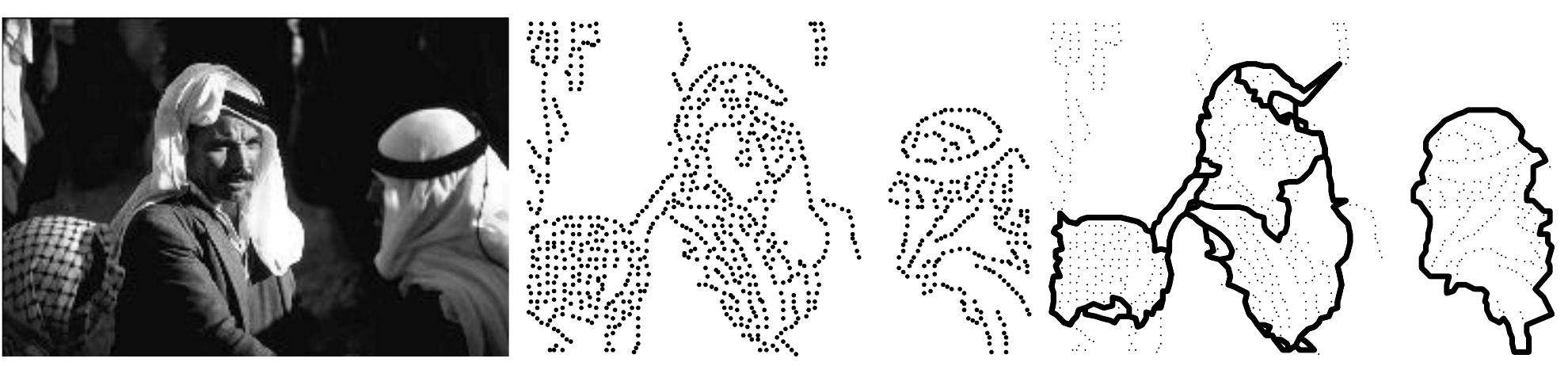}
\includegraphics[width=1.5in]{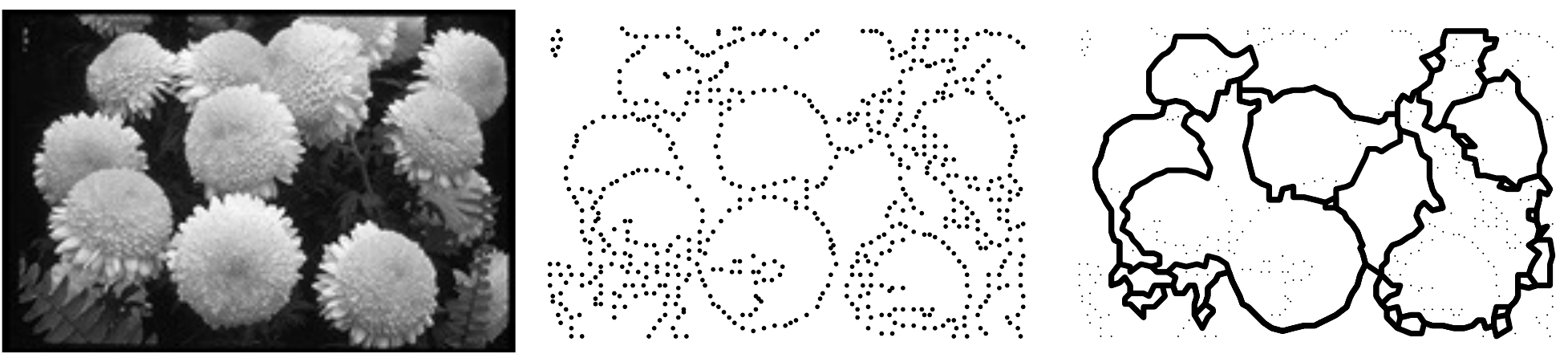}
\includegraphics[width=1.5in]{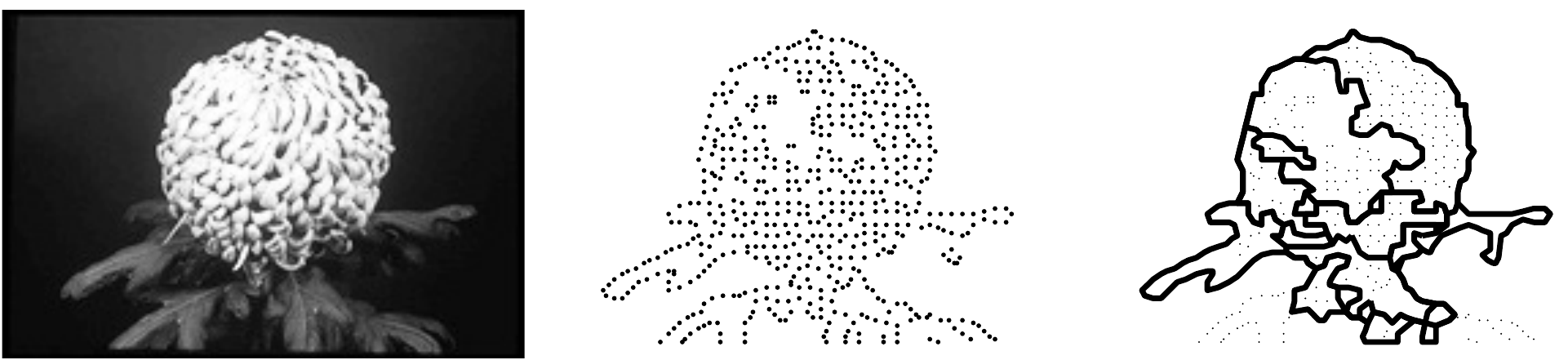}
\includegraphics[width=1.5in]{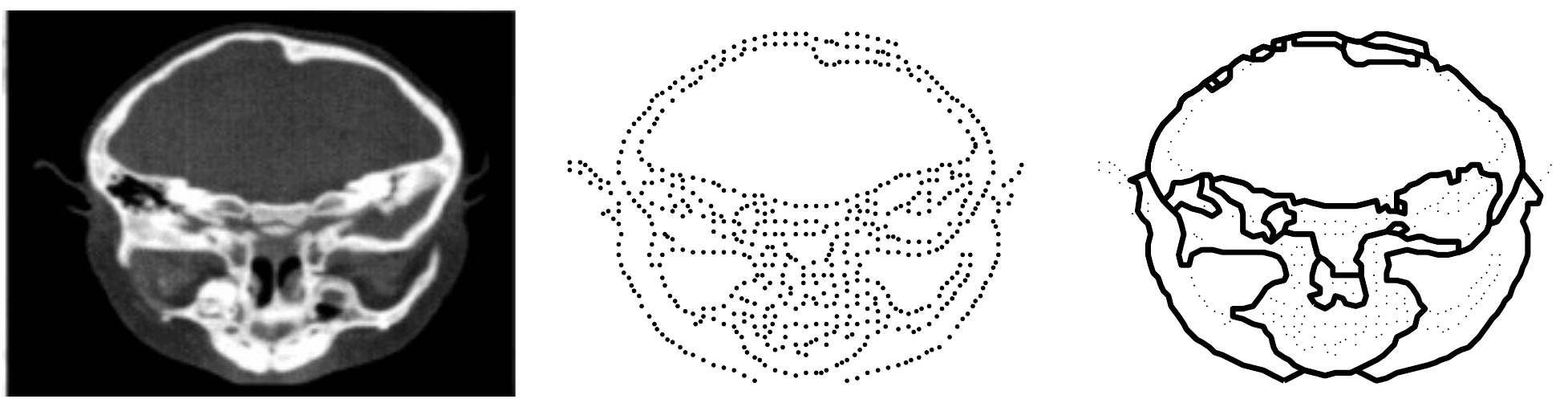}
\includegraphics[width=1.5in]{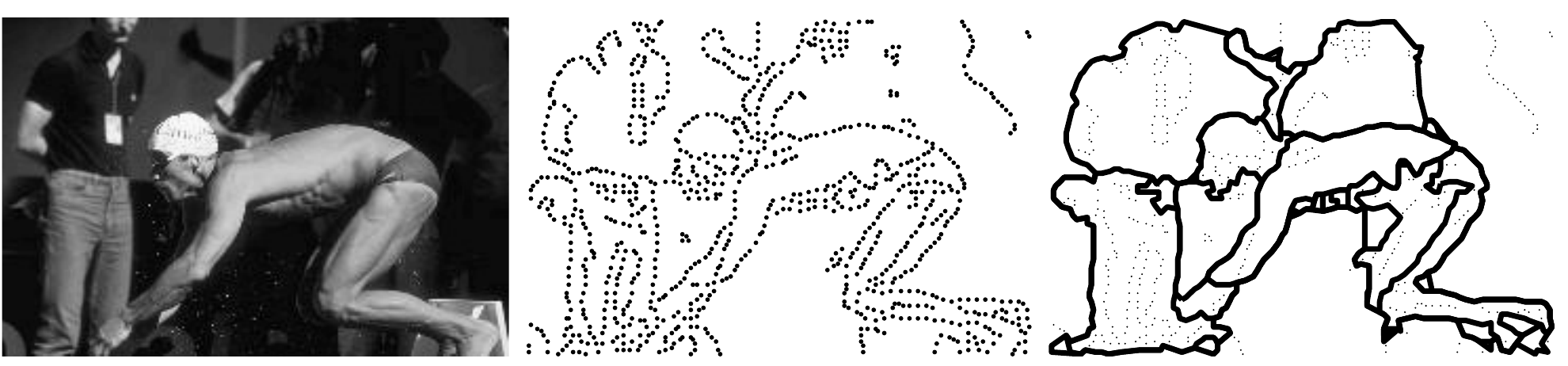}
\includegraphics[width=1.5in]{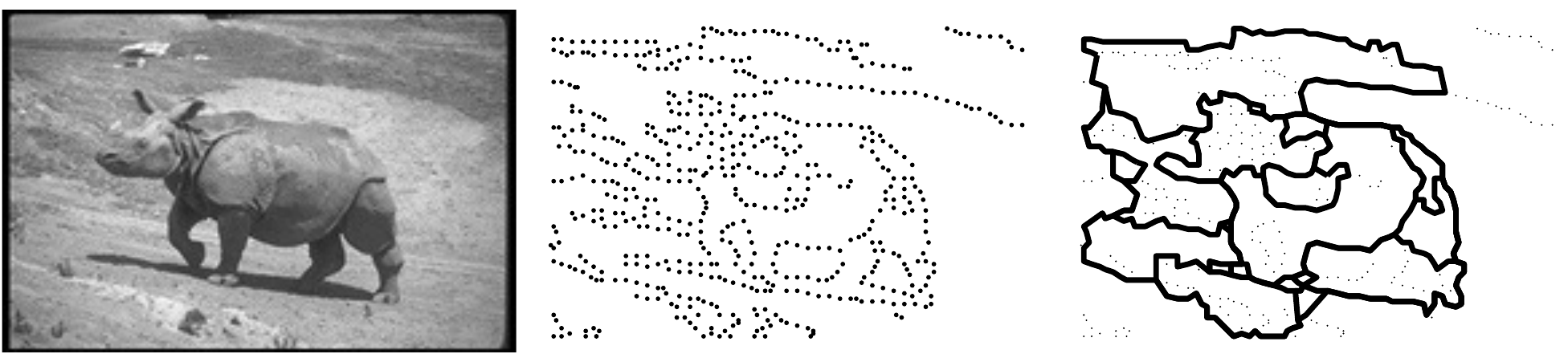}
\includegraphics[width=1.5in]{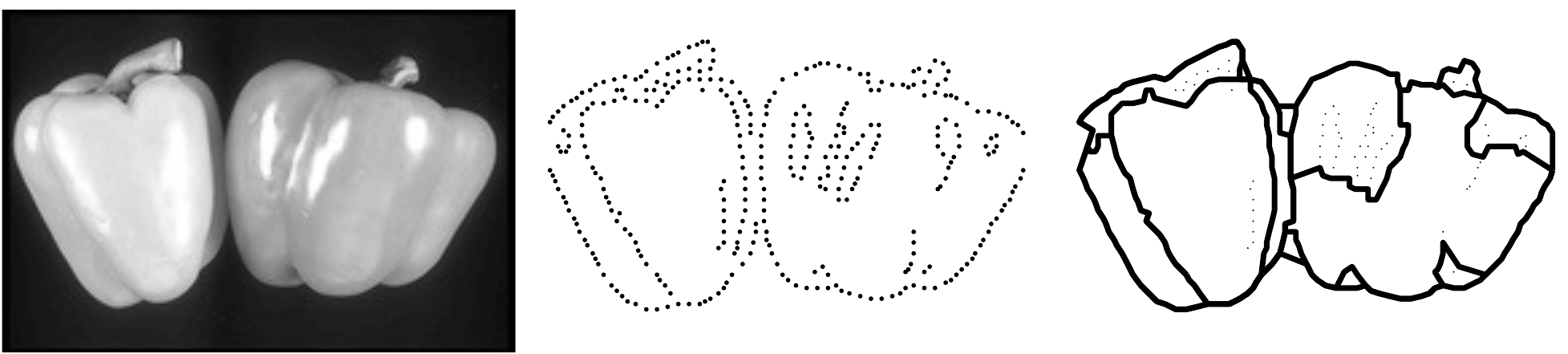}
\includegraphics[width=1.5in]{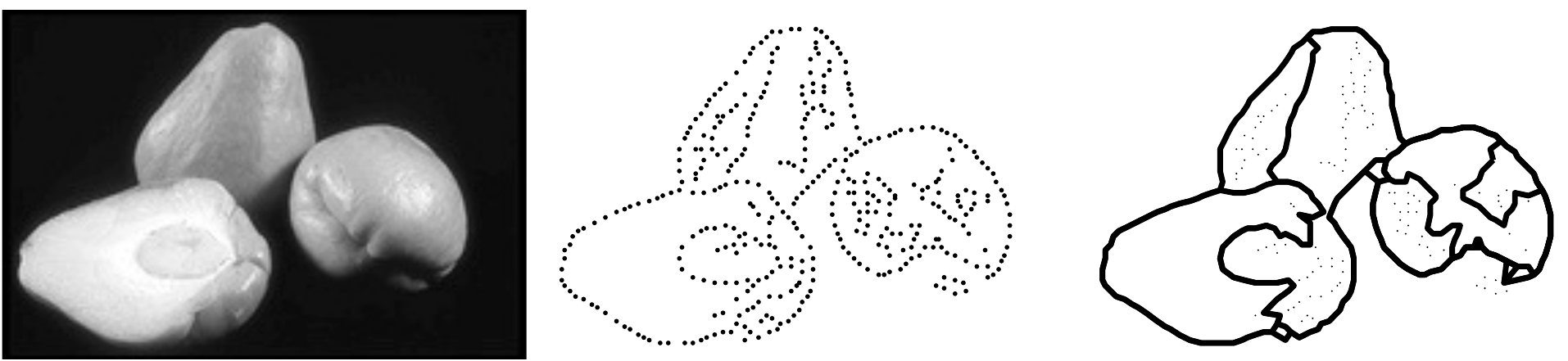}
\includegraphics[width=1.5in]{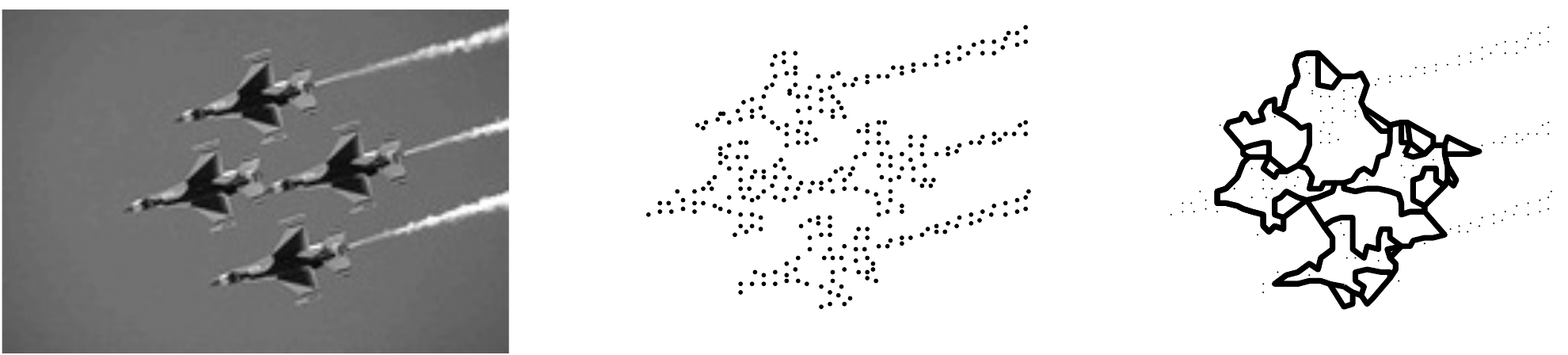}
\includegraphics[width=1.5in]{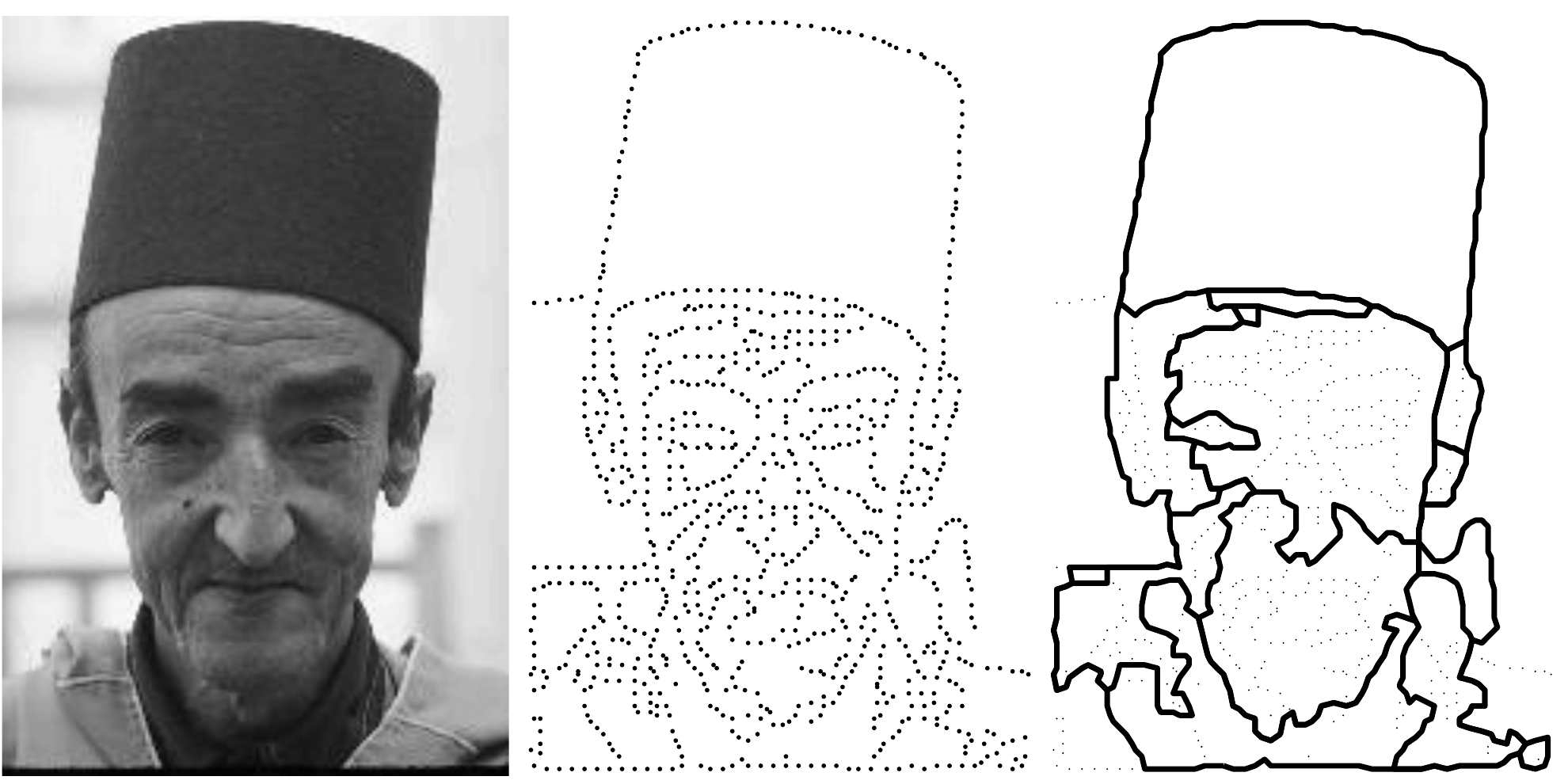}
\includegraphics[width=1.5in]{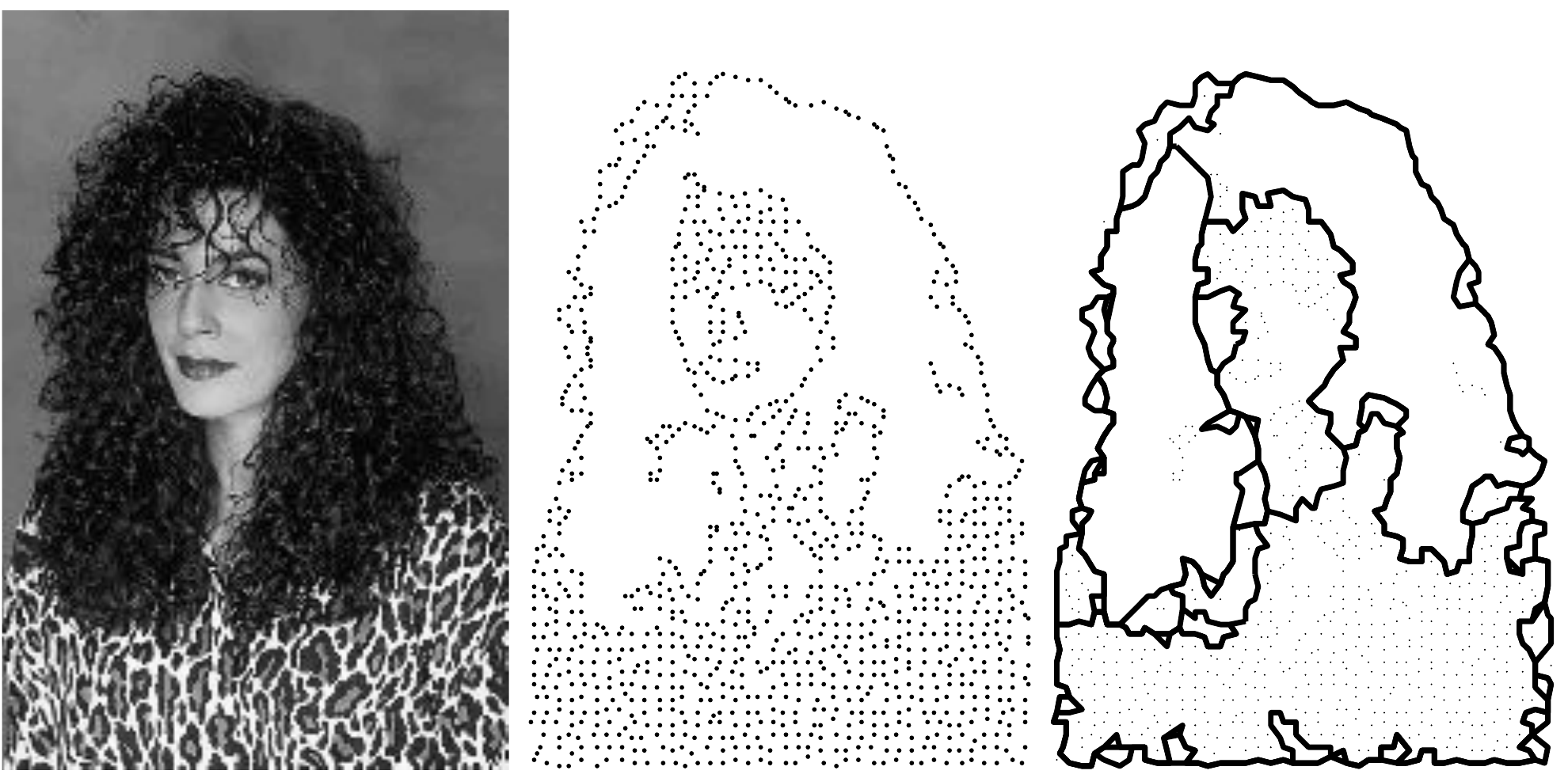}
\includegraphics[width=1.5in]{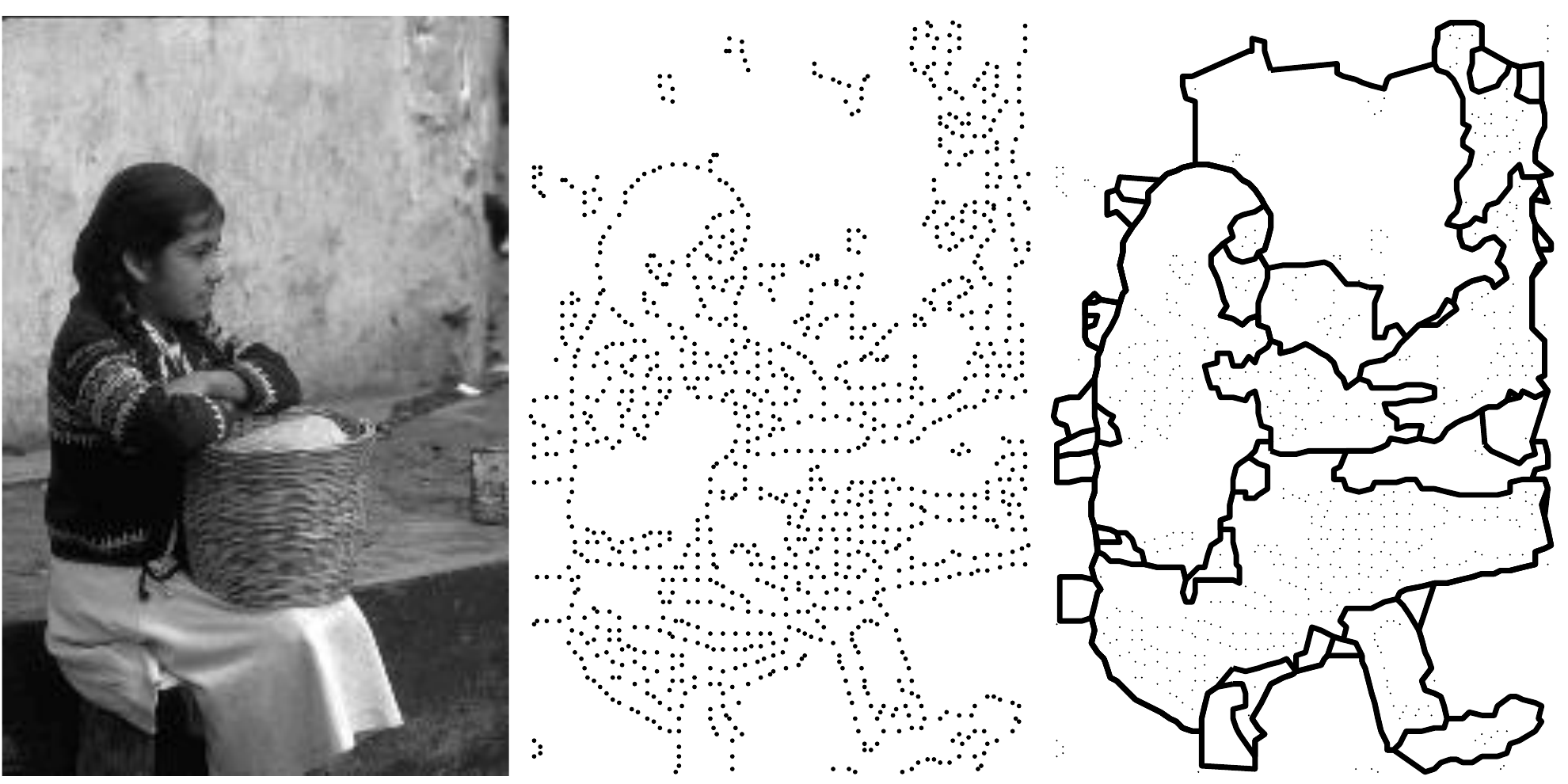}
\includegraphics[width=1.5in]{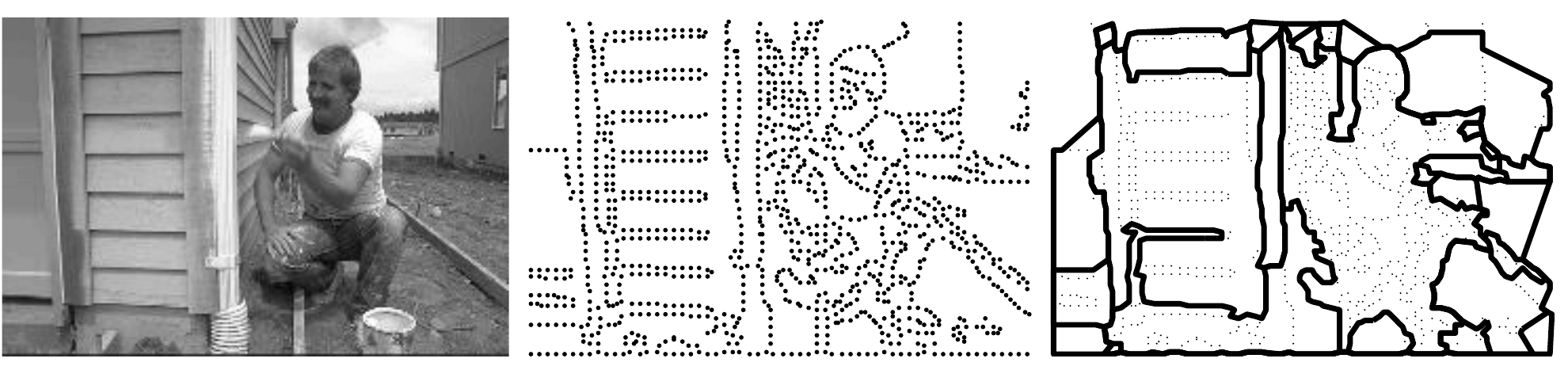}
\includegraphics[width=1.5in]{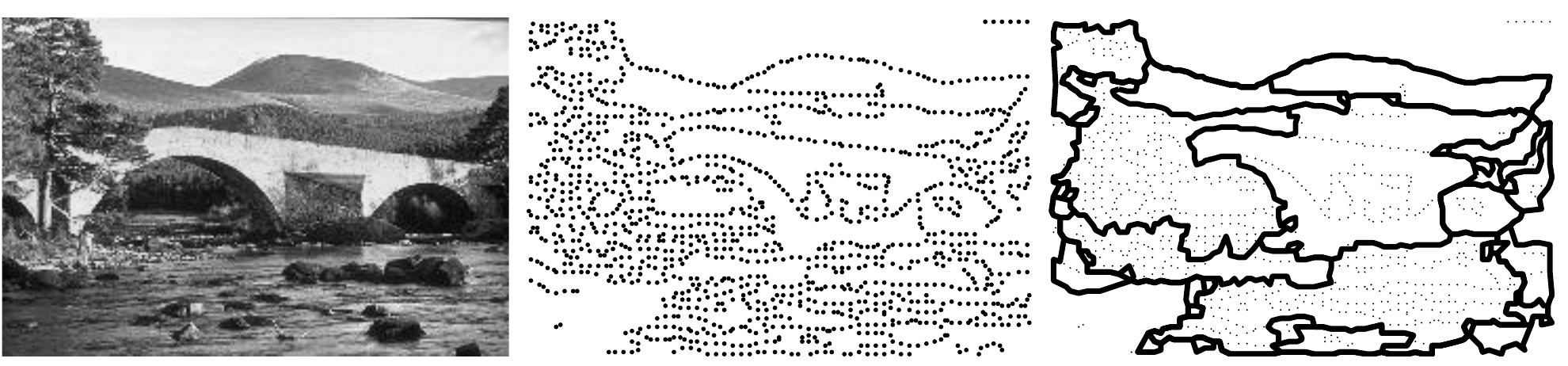}
\includegraphics[width=1.5in]{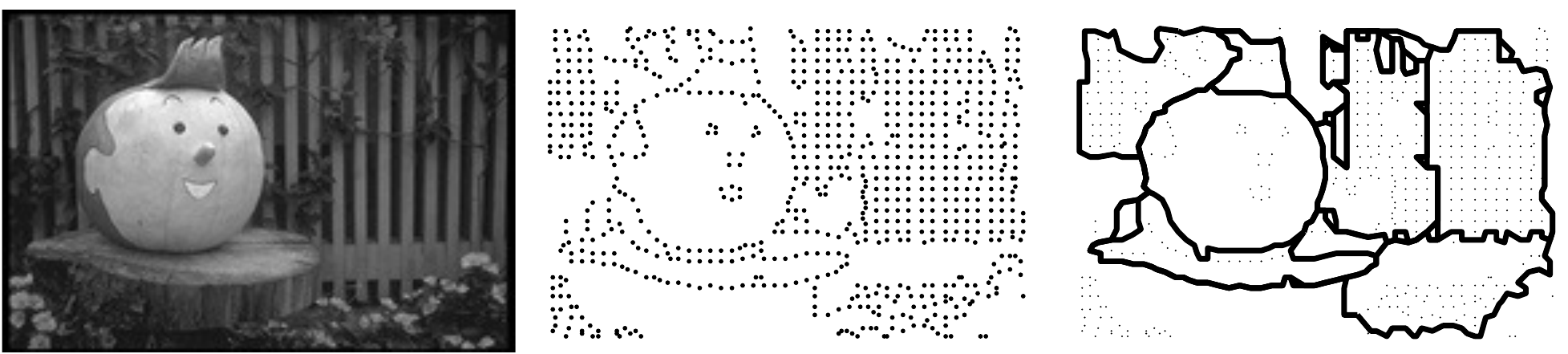}
\includegraphics[width=1.5in]{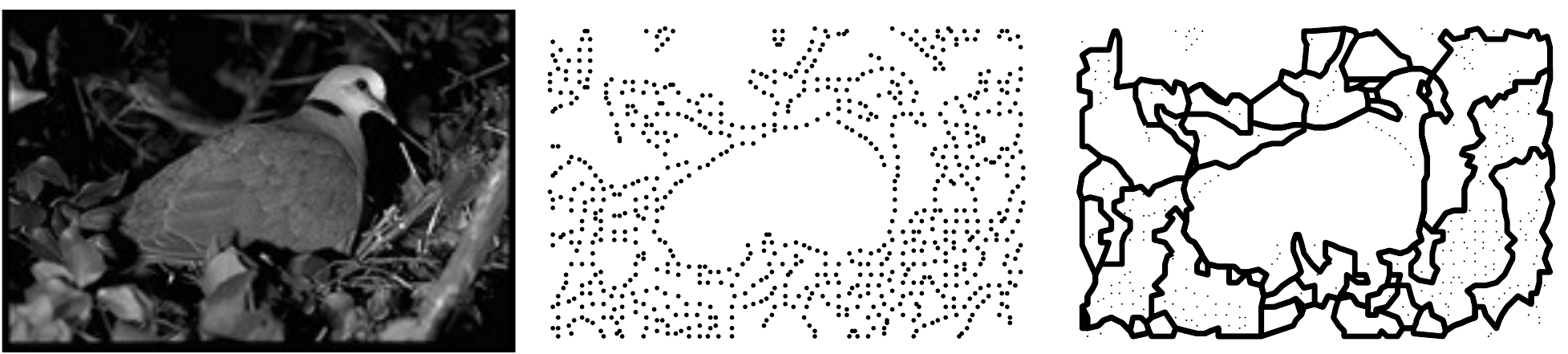}
\includegraphics[width=1.5in]{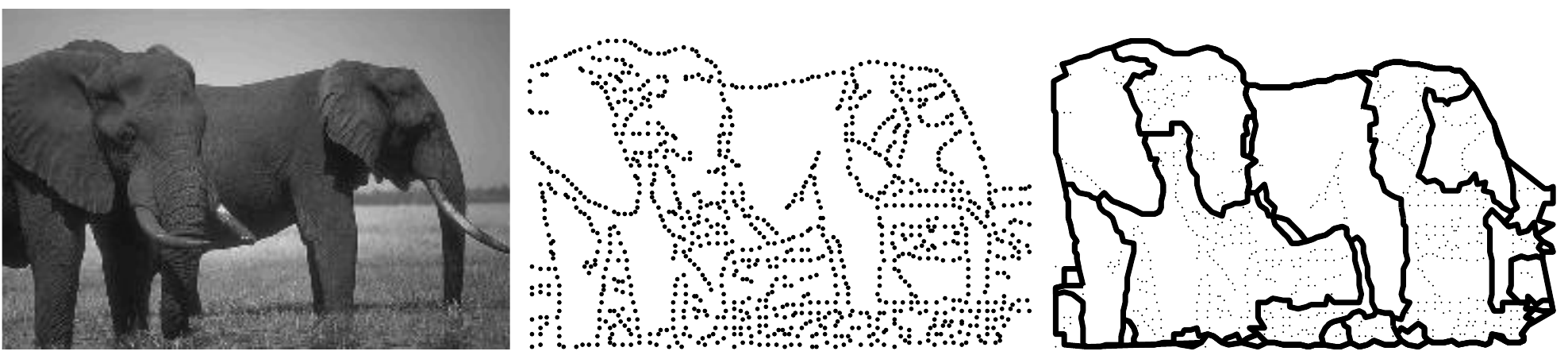}
\includegraphics[width=1.5in]{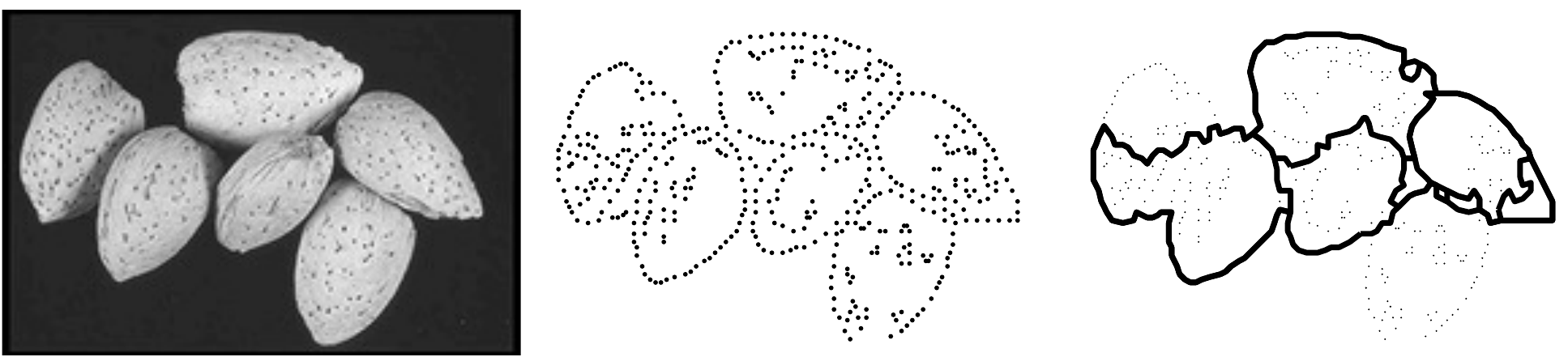}
\includegraphics[width=1.5in]{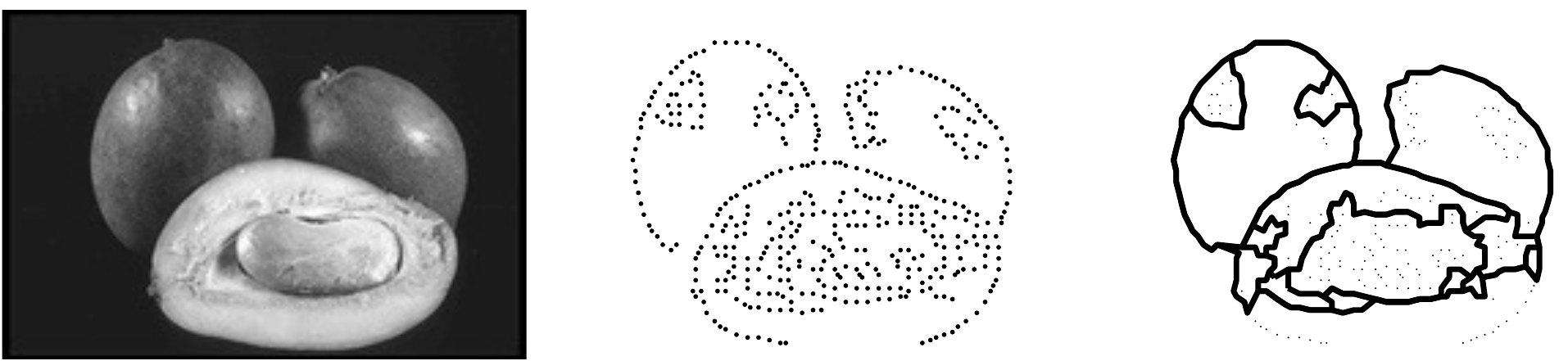}
\includegraphics[width=1.5in]{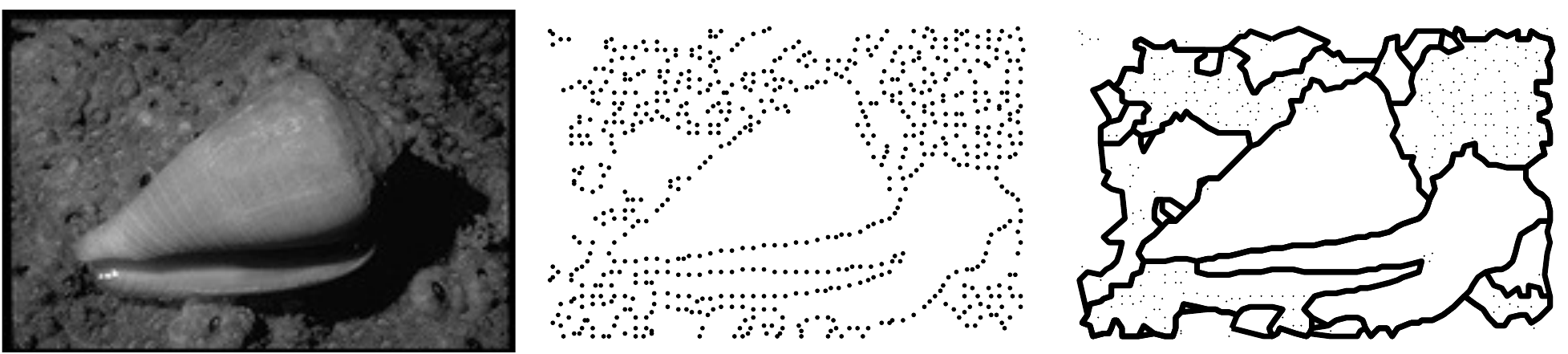}
\includegraphics[width=1.5in]{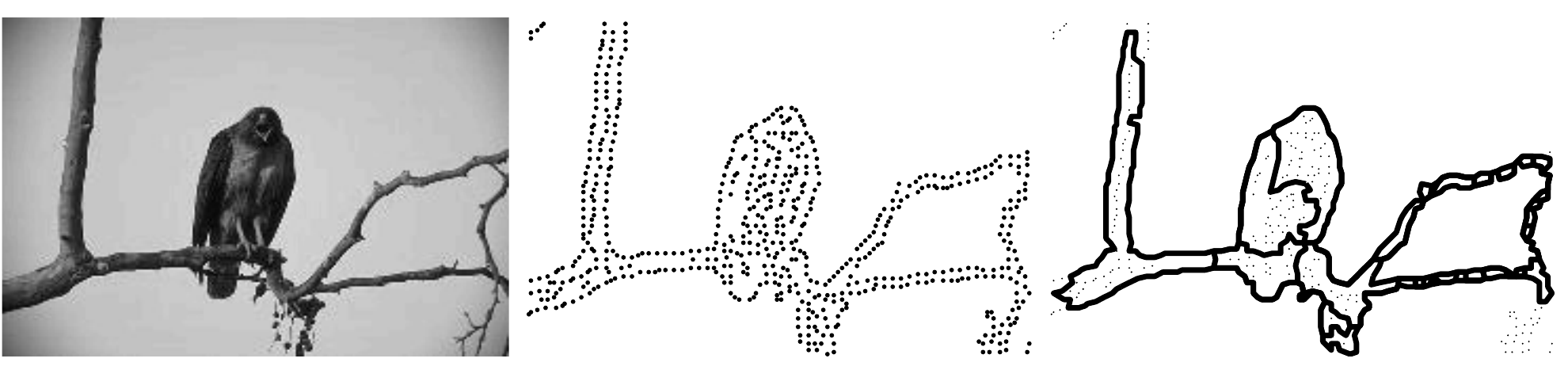}
\includegraphics[width=1.5in]{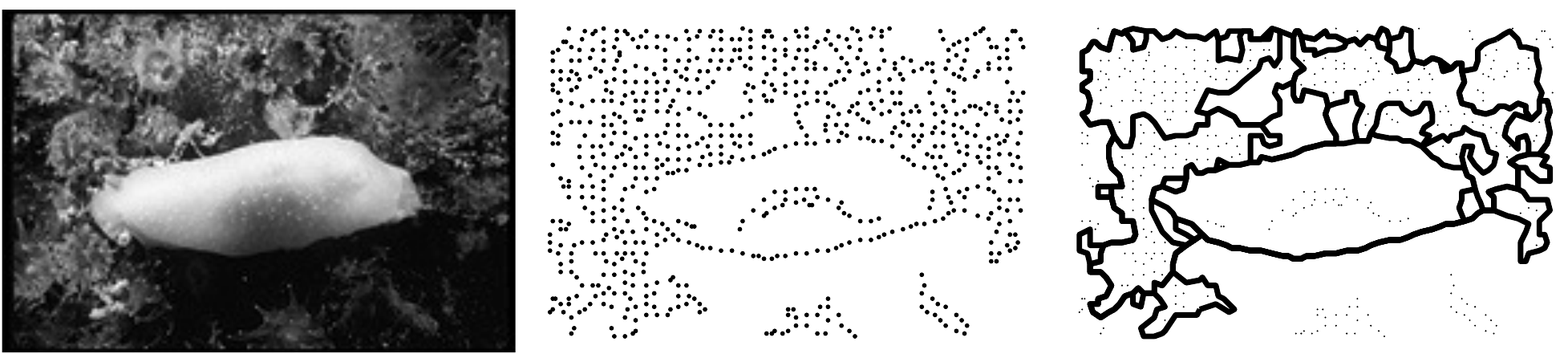}
\includegraphics[width=1.5in]{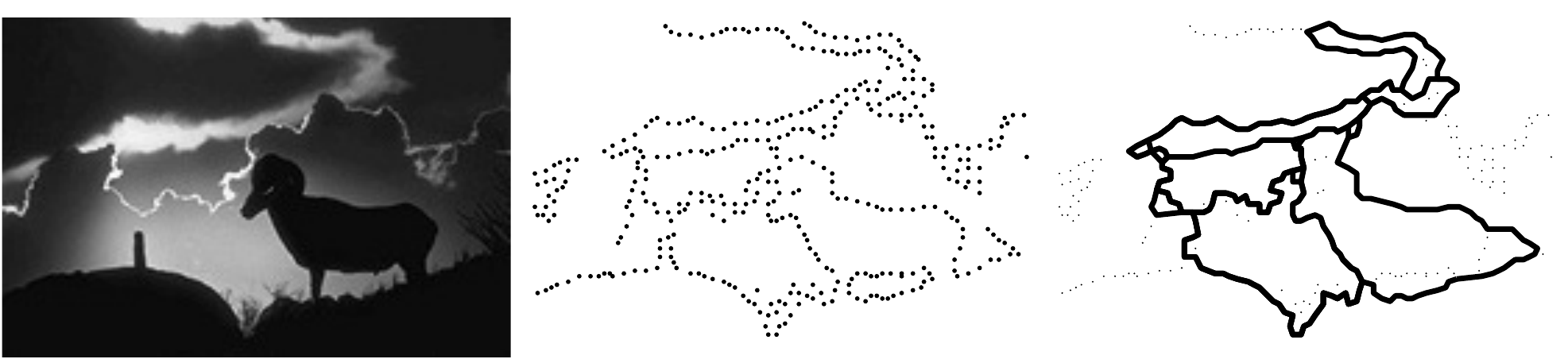}
\includegraphics[width=1.5in]{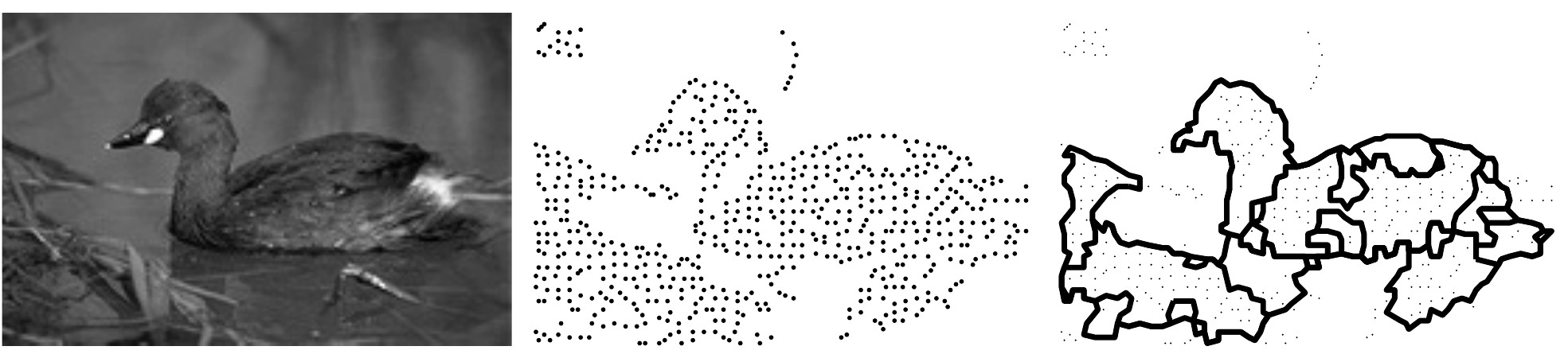}
\includegraphics[width=1.5in]{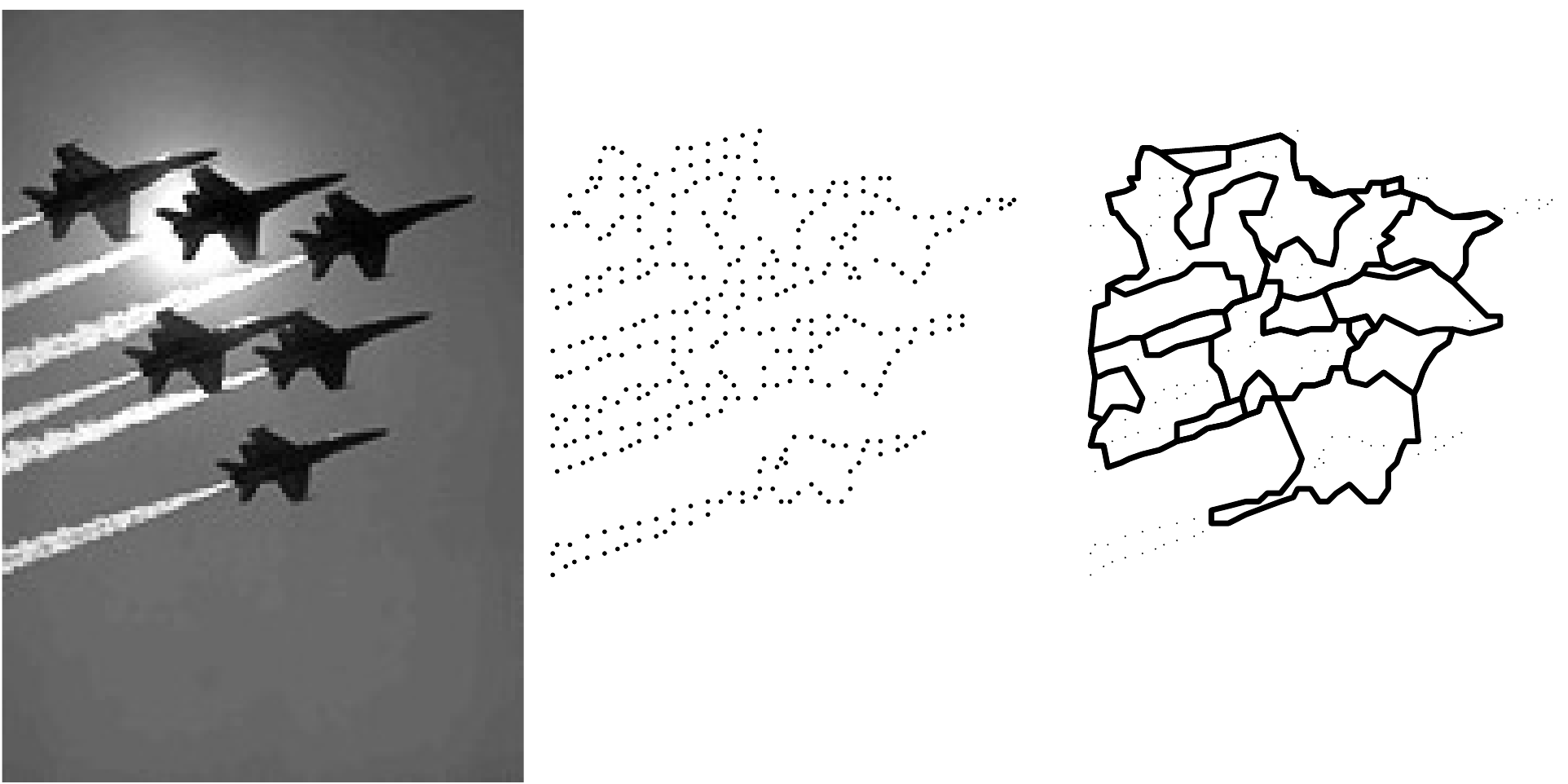}
\includegraphics[width=1.5in]{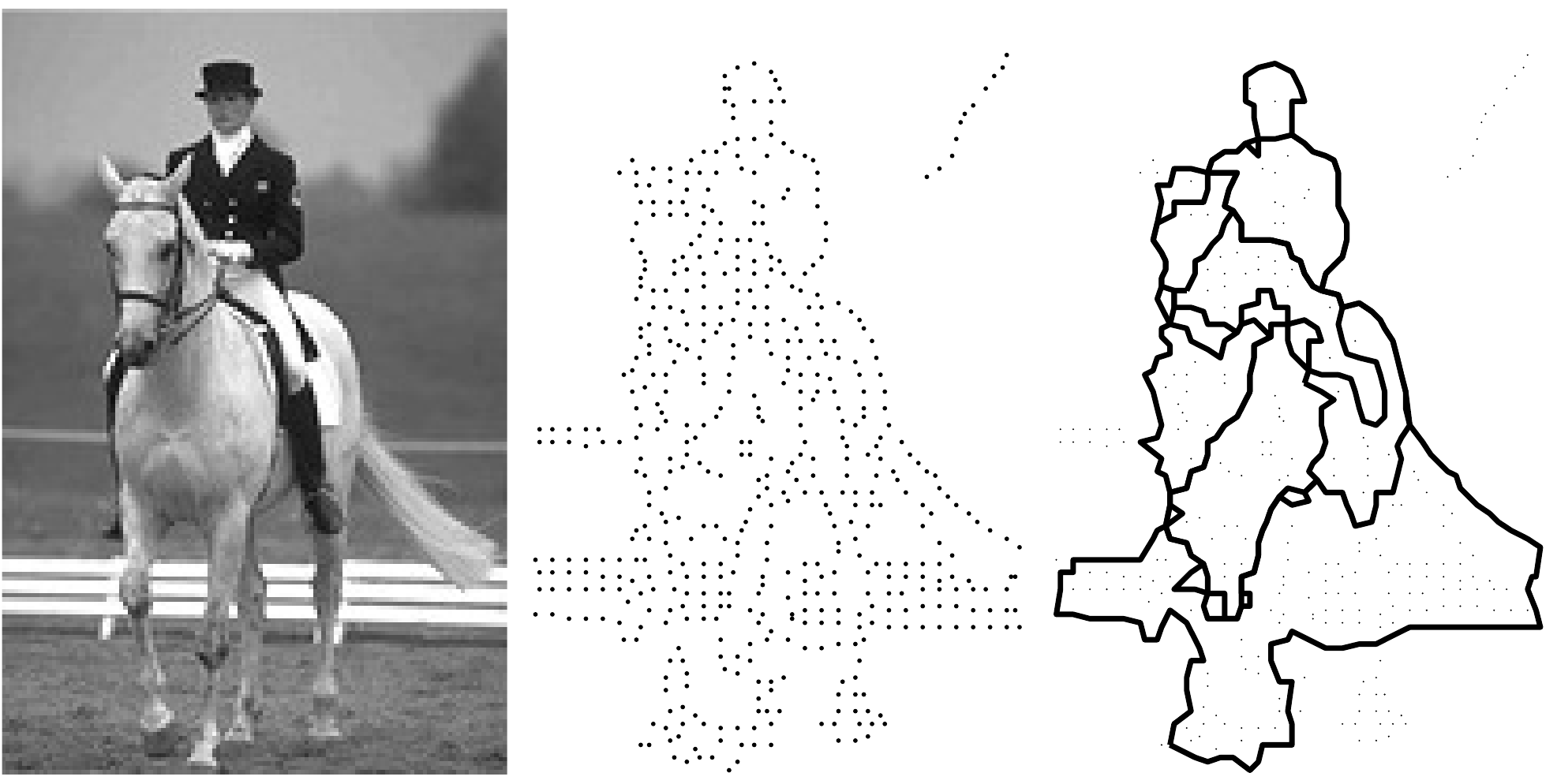}
\includegraphics[width=1.5in]{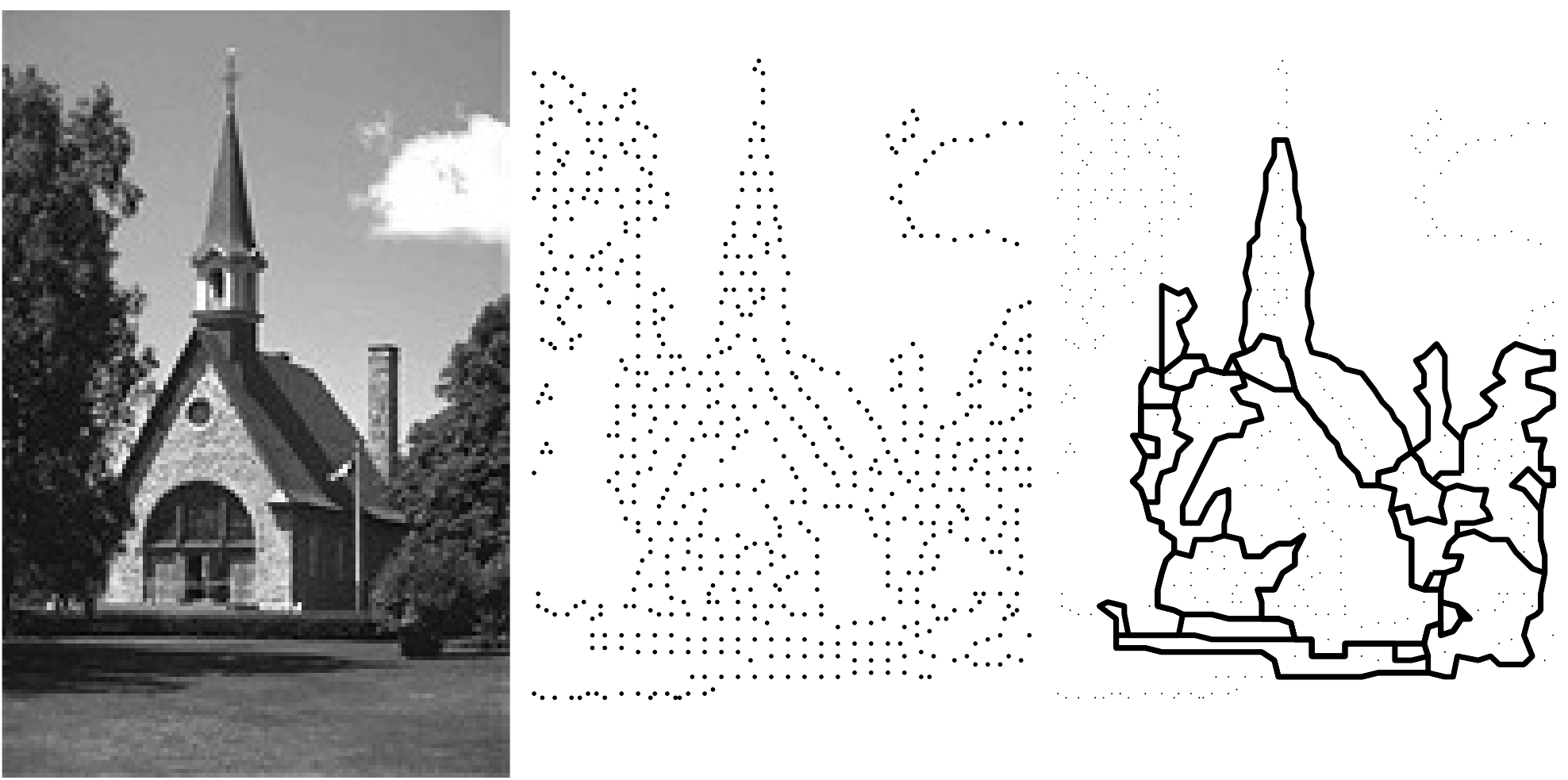}
\includegraphics[width=1.5in]{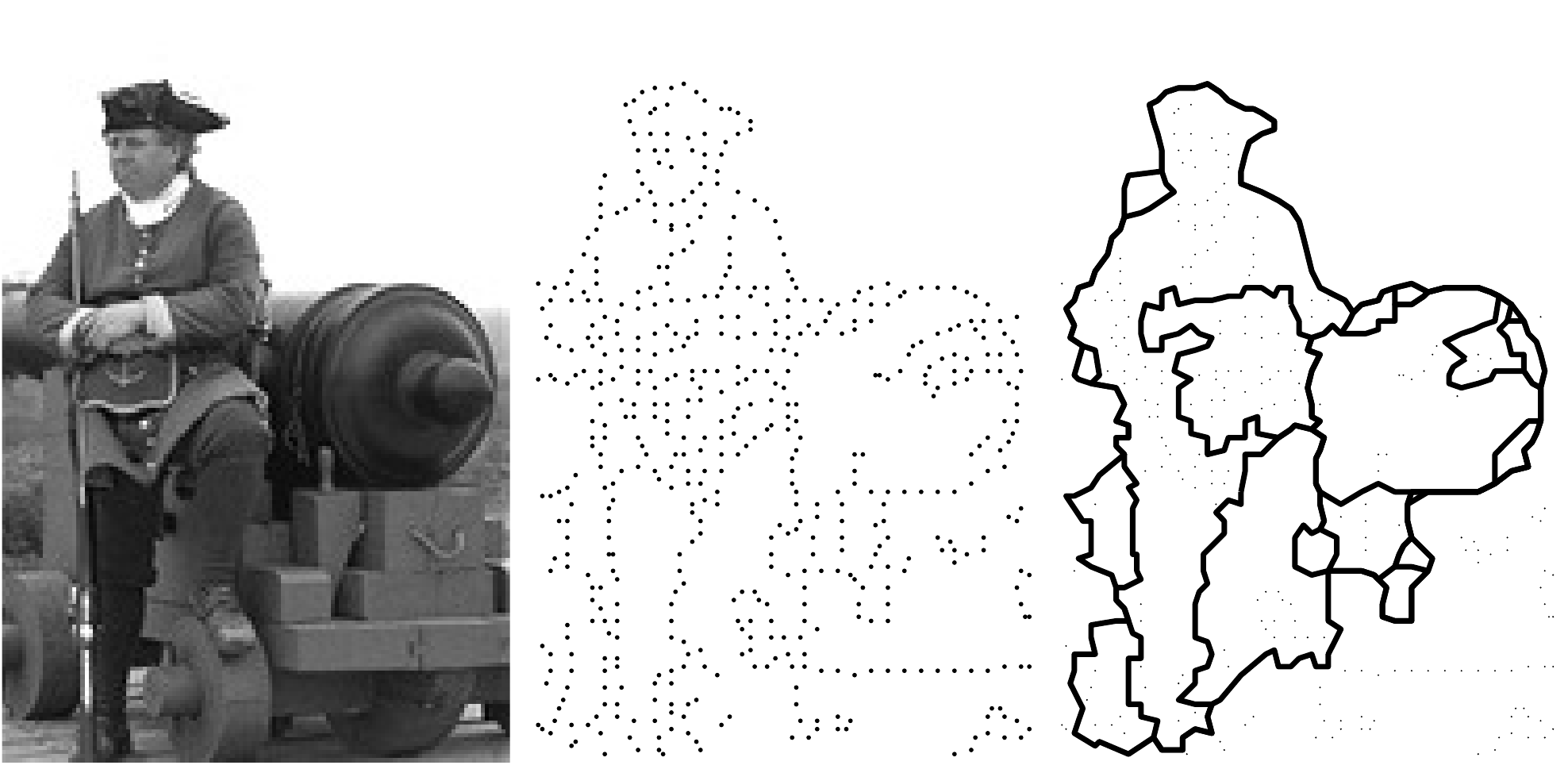}
\includegraphics[width=1.5in]{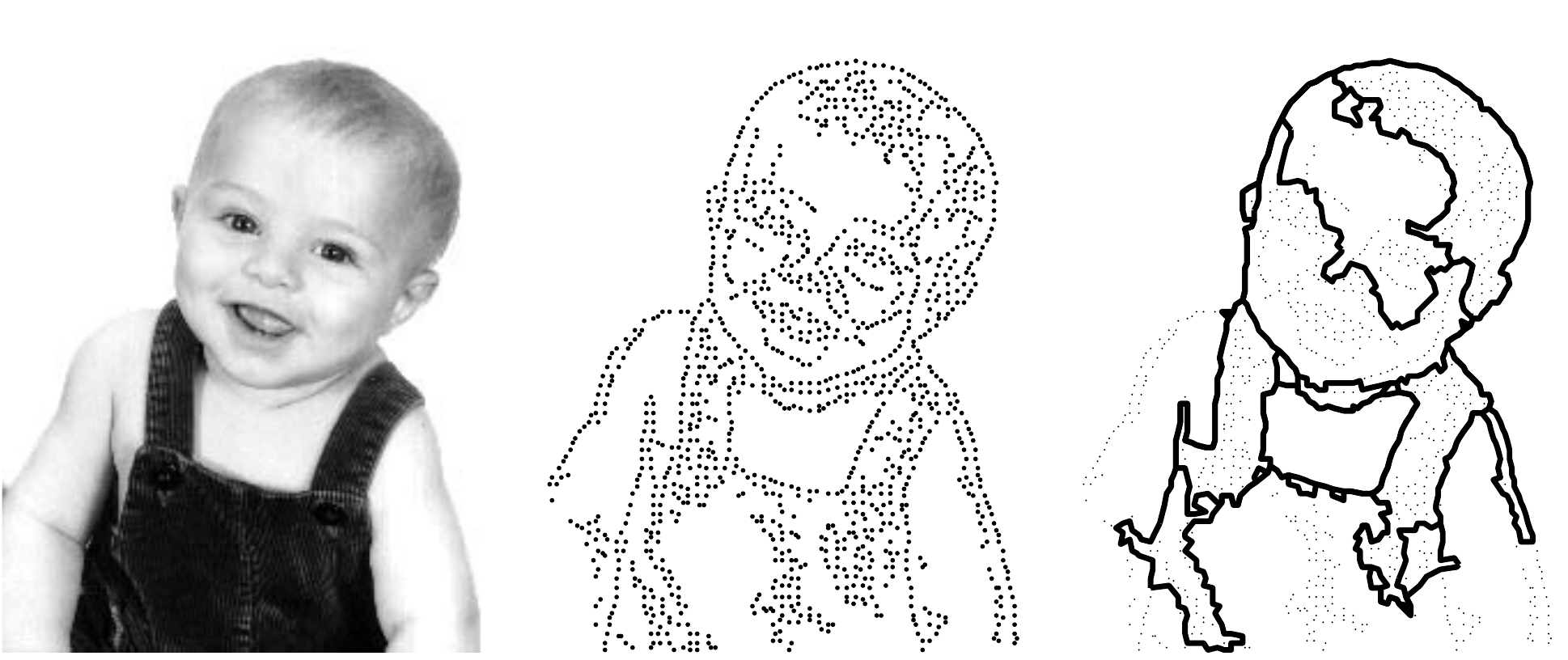}
\includegraphics[width=1.5in]{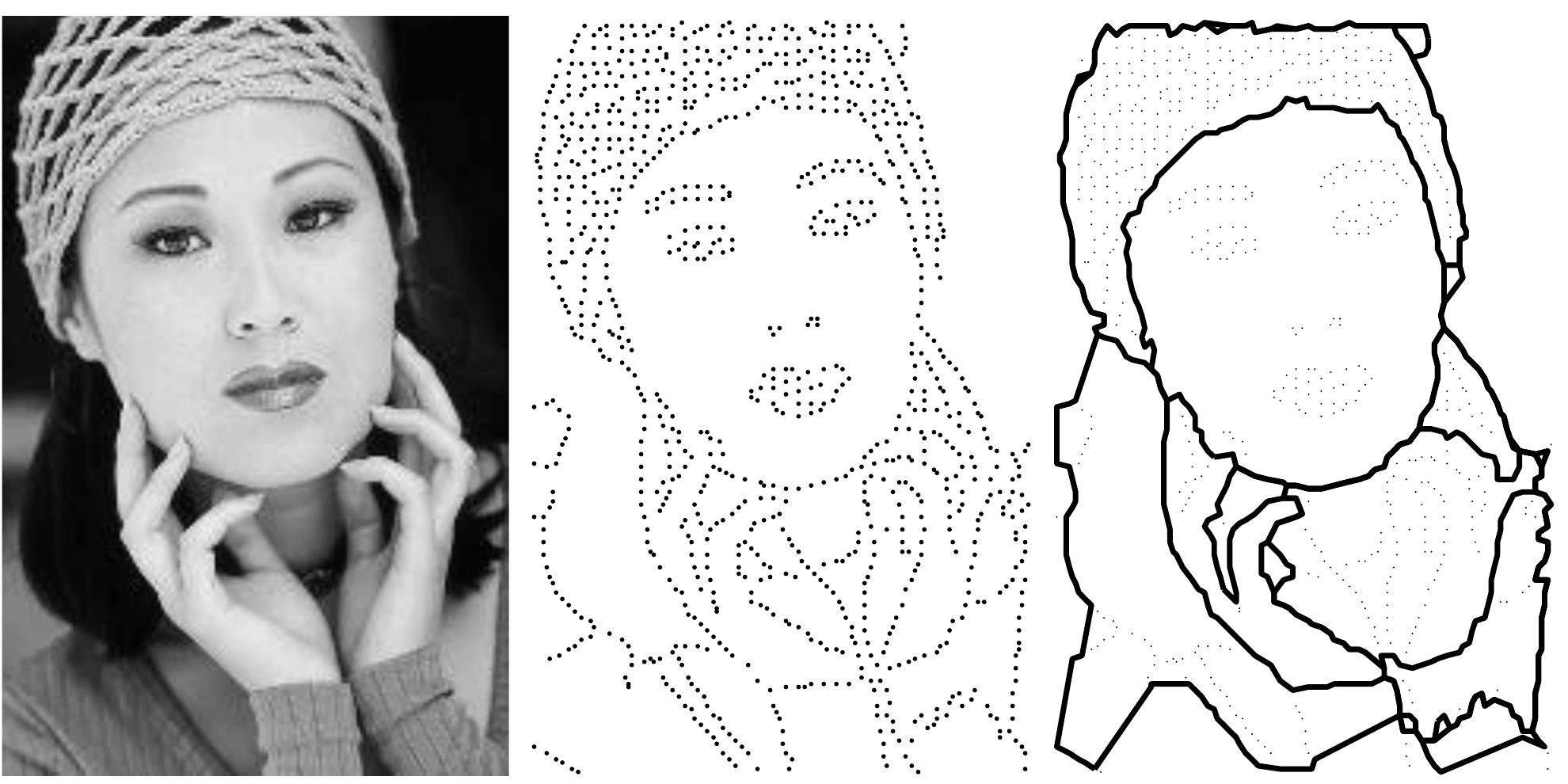}
\caption{Results of applying the grouping algorithm to Canny edge images.
From left to right: input image, dot pattern, five most salient shapes superimposed on the dot pattern.} \label{fig:edgeResults}
\end{figure}

\section{Conclusion}\label{Conclusion}

As our experiments show, the straight polygon representation captures salient shape contained in a dot pattern.
We attribute the effectiveness of the straight polygon representation to the following three characteristics. First, events are at discrete points in time and space. Between the events, polygons retain their shapes and no distortion is introduced. In comparison, the medial axis transform undergoes deformation at continuously in time. This characteristic brings evolution of shapes with less amount of distortion than other smoothing techniques. It also makes a simple recursive implementation and straightforward organization of the events. Second, simple tree type data structure can be constructed to fully capture the evolution of polygons. It can then be used to reverse the deformation and trace a vertex of a polygon to a point in the dot pattern. Third, the curvature dependent smoothing characterized by (\ref{eq:velocity}) finds grouping more quickly than the medial axis transform. See Figure \ref{fig:DilationComparison} for an illustration of comparing the straight skeleton transform and medial axis transform. The thick solid line is an initial minimum spanning tree where the transformation starts. The thick dashed line shows the front of the straight polygon transform when the first split event took place. The thin solid line shows the front of the medial axis transform. Only locations where the fronts are vastly different between the two are shown. This characteristics makes sharp corners important features in defining the grouping.

\begin{figure}
\centering
\includegraphics[width=1.0in]{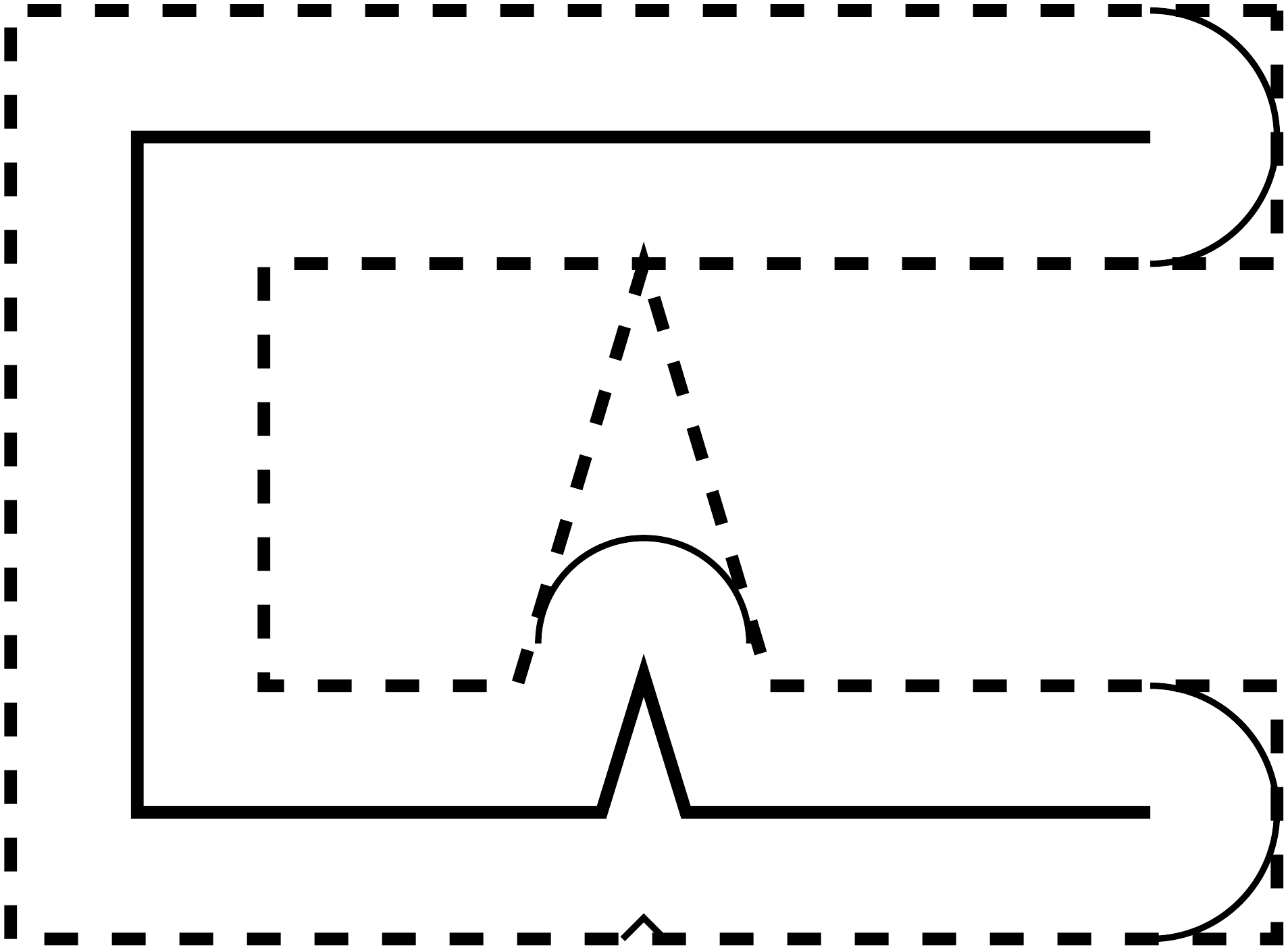}
\caption{Comparison of shape deformation between straight offset polygons and medial axis transform.} \label{fig:DilationComparison}
\end{figure}


Our approach uses a minimum spanning tree and straight polygons as intermediate representations, from which multiple grouping hypotheses are derived by tracing back each polygon vertex to the original point set and extracting disjoint polygons in the trace. Unlike traditional dot grouping algorithms, our approach provides multiple grouping instances. Each dot moves in at least two directions, thus can participate in both foreground and background and provide multiple interpretations. The algorithm is parameter free, simple, and deterministic.
The initial grouping by a minimum spanning tree groups all dots into a single polygon, thus simplifies the subsequent algorithm slightly. However, it brings erroneous grouping especially in thin structures and between two separate clusters of points. More elaborate initial grouping needs to be tried and studied.

A series of polygons derived from the algorithm can be presented in a scale-space manner. When a polygon is split into two, the resulting two polygons can be viewed as having finer scales than their parent. When an outer growing polygon merges multiple polygons into one, the resulting polygon can be viewed as having a coarser scale than the merged polygons. Thus, we can establish easily partial ordering of polygons. How to assign a global scale to the order in a meaningful way requires further investigation.

Both saliency measure of (\ref{eq:convexity}) and matching score of (\ref{eq:matchingScore}) are crude and require more elaborate designs. As we observed in Figure \ref{fig:ShapeDotPatternResults}, the saliency measure often favors rounded shapes over articulated ones. Exclusion of limbs then hampers recognition phase where these parts are visually salient and often serve as distinguish features for effective recognition. The area based matching score also suffers the same shortcoming as limbs do not contribute to the area as much as the body. The shape retrieval experiment conducted for this paper is crude as no geometric transformations are allowed on queries. With improved saliency measure and matching score, the retrieval algorithm should be able to cope with various distortion and a larger scale database of shapes.

\section*{Acknowledgement}
The author thanks Ernest Greene at University of Southern California for providing the shape examples (Figure \ref{fig:DatabaseShapes}). This work was supported by the United States National Science Foundation grants CCF-1117439 and CCF-1421734.

\section*{Appendix}
For completeness, lemmas mentioned in this paper and their proofs are given here.

\begin{lemma}\label{lemma:numEvents}
The number of events (both split and edge) is at most $m-2$ where $m$ is the number of vertices in the initial polygon.
\end{lemma}
\begin{proof}
Let $f(m)$ be the upper bound of the number of events for a polygon with $m$ vertices. We want to show $f(m)\le m-2$. We prove it by induction.
When $m=3$, the only event is of edge type, which collides the three vertices into one. No split event can occur. Thus, $f(3)=1 \le m-2$.
For $m>3$, both split and edge events are possible. If the next event is an edge event, it results in a polygon with $m-1$ vertices. Thus, the number of events is bounded by $f(m-1)+1\le m-2$. If the next event is a split event, it splits the polygon into two. Say the two polygons have $m_a<m$ and $m_b<m$ vertices, respectively. Note that $m_a+m_b = m+1$. Then, $f(m_a) + f(m_b) + 1\le m_a-2 + m_b-2 + 1 = m-2$. Thus, either case, $f(m)\le m-2$.
\end{proof}

\begin{lemma}\label{lemma:numVertices}
The number of vertices in the initial offset polygon derived from a tree of $|V|$ vertices is $2|V|-2+n_L$ where $n_L$ is the number of leaves in the tree.
\end{lemma}
\begin{proof}
At a non-leaf tree node, the number of polygon vertices there is the degree of the tree node. At a leaf tree node, the number of polygon vertices there is 2. Thus,
\begin{equation}
\label{eq:vertices_formula}
m=\sum_{i\in E} d_i + n_L = 2|E|+n_L = 2|V|-2+n_L
\end{equation}
where $d_i$ is the degree of the $i$th vertex. We used the facts that $\sum_{i\in E} d_i=2|E|$ and $|E|=|V|-1$.
\end{proof}

\begin{corollary}\label{corollary:numVertices}
The number of vertices in the initial offset polygon is $\left[2|V|, 3|V|-3\right]$. The lower bound of $2|V|$ occurs when the tree is a path, and the upper bound of $3|V|-3$ occurs when the tree is a star.
\end{corollary}
\begin{proof}
The number of leaves in a tree is between 2 when the tree is a path and $|V|-1$ when the tree is a star. Using them in (\ref{eq:vertices_formula}) gives the bound.
\end{proof}

\bibliography{all}
\bibliographystyle{splncs}
\end{document}